\relax
\documentclass[letterpaper]{article} 
\usepackage{aaai22}  
\usepackage{preamble}  
\usepackage{times}  
\usepackage{helvet}  
\usepackage{courier}  
\usepackage[hyphens]{url}  
\usepackage{graphicx} 
\usepackage{natbib}  
\usepackage{caption} 
\DeclareCaptionStyle{ruled}{labelfont=normalfont,labelsep=colon,strut=off} 
\usepackage{algorithm}
\usepackage{algorithmic}
\usepackage{comment}

\usepackage{newfloat}
\usepackage{listings}
\floatstyle{ruled}
\newfloat{listing}{tb}{lst}{}
\floatname{listing}{Listing}

\newcommand{\ind}{\mathbb I}

\newcount\Includeappendix 
\title{Modeling Attrition in Recommender Systems
with Departing Bandits}
\author{
\author{
Omer Ben-Porat\equalcontrib\textsuperscript{\rm 1},
Lee Cohen\equalcontrib\textsuperscript{\rm 1},
Liu Leqi\equalcontrib\textsuperscript{\rm 2},
Zachary C. Lipton\textsuperscript{\rm 2},
Yishay Mansour\textsuperscript{\rm 1,3}
}}
\affiliations {
    \textsuperscript{\rm 1} Blavatnik School of Computer Science, Tel-Aviv University\\ 
    \textsuperscript{\rm 2} Machine Learning Department, Carnegie Mellon University\\
    \textsuperscript{\rm 3} Google Research\\
}
\usepackage{bibentry}

\begin{document}

\maketitle

\begin{abstract}

Traditionally, when recommender systems are formalized as multi-armed bandits, the policy of the recommender system  influences the rewards accrued, but not the length of interaction. However, in real-world systems, dissatisfied users may depart (and never come back).
In this work, we propose a novel multi-armed bandit setup 
that captures such policy-dependent horizons. 
Our setup consists of a finite set of user \emph{types},
and multiple arms with Bernoulli payoffs.
Each (user type, arm) tuple corresponds 
to an (unknown) reward probability. 
Each user's type is initially unknown 
and can only be inferred 
through their response to recommendations. 
Moreover, if a user is dissatisfied with their recommendation, 
they might depart the system. 
We first address the case where all users share the same type, 
demonstrating that a recent UCB-based algorithm is optimal. 
We then move forward to the more challenging case,
where users are divided among two types. 
While naive approaches cannot handle this setting, 
we provide an efficient learning algorithm 
that achieves $\tilde{O}(\sqrt{T})$ regret,
where $T$ is the number of users.

\end{abstract}

\section{Introduction}
At the heart of online services spanning such diverse industries 
as media consumption, dating, financial products, and more,
recommendation systems (RSs) drive personalized experiences 
by making curation decisions informed 
by each user's past history of interactions.
While in practice, these systems employ
diverse statistical heuristics,
much of our theoretical understanding of them 
comes via stylized formulations 
within the multi-armed bandits (MABs) framework.
While MABs abstract away from many aspects of real-world systems
they allow us to extract crisp insights 
by formalizing fundamental tradeoffs, 
such as that between exploration and exploitation
that all RSs must face~\cite{joseph2016fairness,liu2018incentivizing, patil2020achieving,ron2021corporate}.
As applies to RSs, exploitation consists of 
continuing to recommend items (or categories of items)
that have been observed to yield high rewards in the past,
while exploration consists of recommending items (or categories of items)
about which the RS is uncertain but that could 
\emph{potentially} yield even higher rewards.

In traditional formalizations of RSs as MABs,
the recommender's decisions 
affect only the rewards obtained.
However, real-life recommenders face a dynamic 
that potentially alters the exploration-exploitation tradeoff:
Dissatisfied users have the option to depart the system, never to return.
Thus, recommendations in the service of exploration
not only impact instantaneous rewards 
but also risk driving away users
and therefore can influence long-term cumulative rewards
by shortening trajectories of interactions.

In this work, we propose \emph{departing bandits}
which augment conventional MABs
by incorporating these policy-dependent horizons.
To motivate our setup, 
we consider the following example:
An RS for recommending blog articles
must choose at each time among 
two categories of articles, e.g., 
economics and sports.
Upon a user's arrival, 
the RS recommends articles sequentially. 
After each recommendation, the user 
decides whether to ``click'' the article 
and continue to the next recommendation,
or to ``not click'' and {may }leave the system. 
Crucially, the user interacts with the system 
for a random number of rounds. 
The user's departure probability depends 
on their satisfaction from the recommended item,
which in turn depends on the user's unknown \textit{type}. 
A user's type encodes their preferences 
(hence the probability of clicking) 
on the two topics (economics and sports). 

When model parameters are given, 
in contrast to traditional MABs
where the optimal policy 
is to play the best fixed arm,
departing bandits require 
more careful analysis to derive 
an optimal planning strategy. 
Such planning is a local problem,
in the sense that it is solved for each user. 
Since the user type is never known explicitly
(the recommender must update its beliefs 
over the user types after each interaction),
finding an optimal recommendation policy
requires solving a specific partially observable MDP (POMDP) 
where the user type constitutes the (unobserved) state 
(more details in Section~\ref{subsec:Planning2x2}). 
When the model parameters are unknown,
we deal with a learning problem that is global, 
in the sense that the recommender (learner) 
is learning for a stream of users 
instead of a particular user.

We begin with a formal definition of departing bandits in Section~\ref{sec:problem}, and demonstrate that any fixed-arm policy  
is prone to suffer linear regret. In Section~\ref{sec:ucb-policy}, we establish the UCB-based learning framework used in later sections. We instantiate this framework with a single user type in Section~\ref{sec:oneUser}, where we show that it achieves $\tilde{O}(\sqrt{T})$ regret for $T$ being the number of users.  We then move to the more challenging case with two user types and two recommendation categories in Section~\ref{sec:twoTypes}. To analyze the planning problem, we effectively reduce the search space for the optimal policy by using a closed-form of the \textit{expected return} of any recommender policy. 
These results suggest an algorithm that achieves $\tilde O(\sqrt{T})$ regret in this setting. 
Finally, we show an efficient optimal planning algorithm
for multiple user types and two recommendation categories  (Appendix~\ref{sec:planningManyTypes}) and  describe a 
scheme to construct semi-synthetic problem instances for this setting using real-world datasets  (Appendix~\ref{sec:experimental}).

\medskip
\subsection{Related Work}
\medskip 
MABs have been studied extensively 
by the online learning community \citep{CesaBianchi, Bubeck2012}. 
The contextual bandit literature augments the MAB setup 
with context-dependent rewards
\citep{abbasi2011improved, slivkins2019introduction,Mahadik20DistributedBandits, Korda16, lattimore2020bandit}. %
In contextual bandits, 
the learner observes a \textit{context} 
before they make a decision, 
and the reward depends on the context.
Another line of related work considers 
the dynamics that emerge 
when users act strategically
\cite{Kremer2014,Mansour2015,LeeCohen2019,Bahar2016,bahar2020fiduciary}.
In that line of work, users arriving at the system 
receive a recommendation but act strategically: 
They can follow the recommendation 
or choose a different action. 
This modeling motivates the development 
of incentive-compatible mechanisms as solutions. 
In our work, however, the users are modeled 
in a stochastic (but not strategic) manner.
Users may leave the system if they are dissatisfied with recommendations, 
and this departure follows a fixed (but possibly unknown) stochastic model. 

The departing bandits problem has two important features:  
Policy-dependent horizons, 
and multiple user types 
that can be interpreted as unknown states. 
Existing MAB works \cite{azar2013sequential, cao2020fatigue} 
have addresses these phenomena separately 
but we know of no work that integrates 
the two in a single framework. 
In particular, while \citet{azar2013sequential} 
study the setting with multiple user types,
they focus on a fixed horizon setting.
Additionally, while \citet{cao2020fatigue} 
deal with departure probabilities 
and policy-dependent interaction times for a single user type,
they do not consider the possibility of multiple 
underlying user types. 

The planning part of our problem 
falls under the framework of using Markov Decision Processes 
for modeling recommender-user dynamics~\cite{Shani05}. 
Specifically, our problem works with partially observable user states 
which have also been seen in many recent bandits variants 
\citep{pike2019recovering, leqi2020rebounding}.
Unlike these prior works that focus
on interactions with a single user,
departing bandits consider a stream of users
each of which has an (unknown) type
selected among a finite set of user types.

More broadly, our RS learning problem falls under the domain of reinforcement learning (RL).
Existing RL literature that considers departing users in RSs include~\citet{Zhao2020SSP,Lu2016pomdp,Xiangyu2020}. 
While \citet{Zhao2020SSP} handle users of a single type
that depart the RS within a bounded number of interactions, 
our work deals with multiple user types. 
In contrast to~\citet{Xiangyu2020}, 
we consider an online setting and provide regret guarantees that do not require bounded horizon.
Finally, \citet{Lu2016pomdp} use POMDPs to model user departure and focus on approximating the value function. They conduct an experimental analysis on historical data, while we devise an online learning algorithm with theoretical guarantees.


\section{Departing Bandits}\label{sec:problem}
We propose a new online problem, called \emph{departing bandits}, where the goal is to find the optimal recommendation algorithm for users of (unknown) types,
and where the length of the interactions depends on the algorithm itself.
Formally, the {departing bandits} problem 
is defined by a tuple $\langle [M],[K],\q,\P,\LeavProb \rangle$, 
where $M$ is the number of user \textit{types}, 
$K$ is the number of \textit{categories}, 
$\q\in [0,1]^M$ specifies a prior distribution over types,
and $\P\in (0,1)^{K\times M}$ and $\LeavProb \in (0,1)^{K\times M}$ 
are the \emph{click-probability} 
and the \emph{departure-probability} matrices, respectively.\footnote{We denote by $[n]$ the set $\{1,\dots,n\}$.}

There are $T$ users who arrive sequentially at the RS. At every episode, a new user $t\in [T]$ arrives with a type $type(t)$. We let $\q$ denote the prior distribution over the user types, i.e., $type(t)\sim \q$. Each user of type $x$ \textit{clicks} on a recommended category $a$ 
with probability $\P_{a,x}$. In other words, each click follows a Bernoulli distribution with parameter $\P_{a,x}$. Whenever the user clicks, she stays for another iteration, 
and when the user does not click ({no-click}), she \textit{departs} with probability $\LeavProb_{a,x}$ 
(and stays with probability $1-\LeavProb_{a,x}$). Each user $t$ interacts with the RS (the learner) until she departs.

We proceed to describe the user-RS interaction  protocol. In every iteration $j$ of user $t$, the {learner}
recommends a category $a\in [K]$ to user $t$. 
The user clicks on it with probability $\P_{a,type(t)}$. 
If the user clicks, the learner receives a reward of $r_{t,j}(a)=1$.\footnote{We formalize the reward as is standard in the online learning literature, from the perspective of the learner. However, defining the reward from the user perspective by, e.g., considering her utility as the number of clicks she gives or the number of articles she reads induces the same model.} If the user does not click, the learner receives no reward (i.e., $r_{t,j}(a)=0$), and user $t$ departs with probability $\LeavProb_{a,type(t)}$. 
We assume that the learner knows the value of a  constant $\epsilon>0$
such that $\max_{a,x} \P_{a,x}\leq 1-\epsilon$ (i.e., $\epsilon$ does not depend on $T$). When user $t$ departs, she does not interact with the learner anymore 
(and the learner moves on to the next user $t+1$). 
For convenience, the departing bandits problem protocol is summarized in Algorithm \ref{alg-protocol}.

{
\begin{algorithm}[tb]
\textbf{Input}:  number of types $M$, number of categories $K$, and number of users (episodes) $T$  \\
\textbf{Hidden Parameters}: types prior~$\q$, click-probability~$\mathbf P$, and departure-probability~$\LeavProb$
\begin{algorithmic}[1] %
\FOR{episode $t\leftarrow 1,\dots,T$} 
\STATE{a new user with type $type(t)\sim \q$ arrives} 
\STATE $j\leftarrow 1$, 
$depart \leftarrow false$ 
\WHILE{$depart$ is $false$}
\STATE the learner picks a category $a\in [K]$
\STATE with probability $\mathbf P_{a,x}$, user $t$ clicks on $a$ and $r_{t,j}(a)\leftarrow 1$; otherwise, $r_{t,j}(a)\leftarrow 0$
\IF {$r_{t,j}(a)=0$}
 \STATE with probability $\LeavProb_{a,x}$:  $depart\leftarrow true$ and user $t$ departs
\ENDIF
\STATE the learner observes $r_{t,j}(a)$ and $depart$
 \IF{$depart$ is $false$} 
 \STATE {$j\leftarrow j+1$} 
 \ENDIF
\ENDWHILE
\ENDFOR
\caption{The Departing Bandits  
Protocol}\label{alg-protocol}
\end{algorithmic}
\end{algorithm}
}

Having described the protocol, we move on to the goals 
of the learner. Without loss of generality, we assume that the online learner's recommendations
are made based on a \textit{policy} $\pi$, which is a mapping from the history of previous interactions (with that user) to recommendation categories. 
 
For each user (episode) $t\in[T]$, the learner selects a policy $\pi_t$ that recommends category $\pi_{t,j}\in[K]$ at every iteration $j\in [N^{\pi_t}(t)]$, where $N^{\pi_t}(t)$ denotes the episode length (i.e., total number of iterations policy $\pi_t$ interacts with user $t$ until she departs).\footnote{We limit the discussion to deterministic policies solely; this is w.l.o.g. (see Subsection \ref{subsec:Planning2x2} for further details).} The \textit{return} of a policy $\pi$, denoted by $V^{\pi}$ is the cumulative reward the learner  obtains when executing the policy $\pi$ until the user departs. Put differently, the return of $\pi$ from user $t$ is the random variable  $V^{\pi}=\sum_{j=1}^{N^\pi(t)}r_{t,j}(\pi_{t,j})$.

We denote by $\pi^*$ an optimal policy, namely a policy that maximizes the expected return, $\pi^*=\argmax_{\pi}  \E[V^{\pi}]$. Similarly, we denote by $V^*$ the optimal return, i.e., $V^*=V^{\pi^*}$.

We highlight two algorithmic tasks. The first is the planning task, in which the goal is to find an optimal policy $\pi^*$, given $\P, \LeavProb, \q$. The second is the online learning task. We consider settings where the learner knows the number of categories, $K$, the number of types, $M$, and the number of users, $T$, but has no prior knowledge regarding $\P,\LeavProb$ or $\q$. In the online learning task, the \textit{value} of the learner's algorithm is the sum of the returns obtained from all the users, namely
\[
\sum_{t=1}^{T} V^{\pi_t} = \sum_{t=1}^{T} \sum_{j=1}^{N^\pi(t)}r_{t,j}(\pi_{t,j}).
\]
The performance of the leaner is compared to that of the best policy, formally defined by the \textit{regret} for $T$ episodes,
\begin{equation}\label{eq:regredef}%
R_T= T\cdot \E[V^{\pi^*}] - \sum_{t=1}^{T}V^{\pi_t}.    
\end{equation}
The learner's goal is to minimize the expected regret $\E[R_T]$.

\subsection{Example}\label{subsec:example}
\begin{table}[t]
    \centering
    \begin{tabular}
        {|p{2cm}||p{2cm}|p{2cm}|}
 \hline
   & Type $x$  & Type $y$\\
 \hline
 Category $1$   & $\P_{1,x}=0.5$ & $\P_{1,y}=0.28$ \\ 
 Category $2$ &  $\P_{2,x}=0.4$ &  $\P_{2,y}=0.39$ \\ 
 \hline
 Prior &$\q_x=0.4$ & $\q_y=0.6$\\
 \hline
    \end{tabular}
    \caption{The departing bandits instance in Section~\ref{subsec:example}.
    }  
    \label{tab:exampleinmodel}
\end{table}

The motivation for the following example is two-fold. First, to get the reader acquainted with our notations; and second, to show why fixed-arm policies are inferior in our setting. 

Consider a problem instance with two user types ($M=2$), which we call $x$ and $y$ for convenience. There are two  categories ($K=2$), and given no-click the departure is deterministic, i.e., $\LeavProb_{a,\tau}=1$ for every category $a\in[K]$ and type $\tau \in [M]$. That is, every user leaves immediately if she does not click.  Furthermore, let the click-probability $\P$ matrix
and the user type prior distribution $\q$ be as in Table~\ref{tab:exampleinmodel}.

Looking at $\P$ and $\q$, we see that Category 1 is better for Type $x$, while Category 2 is better for type $y$. Notice that without any additional information, a user is more likely to be type $y$. Given the prior distribution, recommending Category $1$ in the first round yields an expected reward of $\q_x\P_{1,x}+\q_y\P_{1,y}=0.368$. Similarly, recommending Category $2$ in the first round results in an expected reward of $0.394$. Consequently, if we recommend \textit{myopically}, i.e., without considering the user type, always recommending Category $2$ is better than always recommending Category 1.

Let $\pi^a$ denote the fixed-arm policy that always selects a single category $a$. Using the tools we derive in Section~\ref{sec:twoTypes} and in particular Theorem~\ref{thm:valueUpperBound}, we can compute the expected returns of $\pi^1$ and $\pi^2$, $\E[V^{\pi^1}]$ and $\E[V^{\pi^2}]$. Additionally, using results from Section~\ref{subsec:2x2Characterizing}, we can show that the optimal policy for the planning task, $\pi^*$, recommends Category $2$ until iteration $7$, and then recommends Category $1$ for the rest of the iterations until the user departs. 

Using simple calculations, we see that $\E[V^{{{\pi^*}}}] - \E[V^{\pi^1}]> 0.0169$ and 
$\E[V^{{{\pi^*}}}] - \E[V^{\pi^2}] > 1.22 \times 10^{-5}$; hence,  the expected return of the optimal policy is greater than the returns of both fixed-arm policies by a constant. As a result, if the learner only uses fixed-arm policies ($\pi^a$ for every $a\in[K]$), she suffers linear expected regret, i.e., $\E[R_T]=T\cdot \E[V^{\pi^*}] - \sum_{t=1}^{T}\E[V^{\pi^a}]=\Omega(T).$

\section{UCB Policy for Sub-exponential Returns}\label{sec:ucb-policy}
In this section, we introduce the learning framework used in the paper and provide a general regret guarantee for it. 

In standard MAB problems, at each $t\in [T]$ the learner picks a single arm and receives a single  sub-Gaussian reward. In contrast, in departing bandits, at each $t\in [T]$ the learner receives a return $V^\pi$, which is the cumulative reward of that policy. 
The return $V^\pi$ depends on the policy $\pi$ not only through the obtained rewards at each iteration but also through the total number of iterations (trajectory length). 
Such returns are not necessarily sub-Gaussian. Consequently, we cannot use standard MAB algorithms as they usually rely on concentration bounds for sub-Gaussian rewards. Furthermore, as we have shown in Section~\ref{subsec:example},  %
in departing bandits fixed-arm policies can suffer linear regret (in terms of the number of users), which suggests considering a more expressive set of policies. 
This in turn yields another disadvantage for using MAB algorithms for departing bandits, 
as their regret is linear in the number of arms (categories) $K$.

As we show later in Sections~\ref{sec:oneUser} and \ref{sec:twoTypes}, for some natural instances of the departing bandits problem, the return from each user is sub-exponential (Definition~\ref{def:subExp}). Algorithm~\ref{alg-policy-UCB}, which we propose below, receives a set of policies $\Pi$ as input, along with other parameters that we describe shortly. The algorithm is a restatement of the \textit{UCB-Hybrid} Algorithm from \citet{JIA2021}, with two modifications: (1) The input includes a set of policies rather than a set of actions/categories, and accordingly, the confidence bound updates are based on return samples (denoted by $\hat{V}^\pi$) rather than reward samples. (2) There are two global parameters ($\tilde{\tau}$ and $\eta$) instead of two local parameters per action. If the return from each policy in $\Pi$ is sub-exponential, Algorithm~\ref{alg-policy-UCB} not only handles sub-exponential returns, but also comes with the following guarantee: Its expected value is close to the value of the best policy in $\Pi$.  

\subsection{Sub-exponential Returns}
For convenience, we state here the definition of sub-exponential random variables \cite{Eldar2012Book}.

\begin{definition}\label{def:subExp}
We say that a random variable $X$ 
is \emph{sub-exponential} with parameters $(\tau^2,b)$
if for every $\gamma$ such that $|\gamma|<1/b$,
\[
\E[\exp({\gamma(X-\E[X])})]\leq \exp(\frac{\gamma ^2\tau^2}{2}).
\]
In addition, for every $(\tau^2,b)$-sub-exponential random variables,  there exist constants $C_1,C_2>0$ 
such that the above is equivalent 
to each of the following properties: 
\begin{enumerate}
      \item\label{it:tails} Tails: $\forall v \geq 0: \Pr[|X| > v]
     \leq  \exp({1-\frac{v}{C_1}})$.
    \item\label{it:moments} Moments: $
    \forall p \geq 1:  (\E[|X|^p])^{1/p} \leq C_2p$.
   
\end{enumerate}
\end{definition}

Let $\Pi$ be a set of policies with the following property: There exist $\tilde \tau,\eta$ such that the return of every policy $\pi\in\Pi$ is ($\tau^2,b$)-sub-exponential with $\tilde \tau \geq \tau$ and $\eta \geq \frac{b^2}{\tau^2}$. The following Algorithm~\ref{alg-policy-UCB} receives as input a set of policies $\Pi$ with the associated parameters, $\tilde \tau$ and $\eta$. Similarly to the UCB algorithm, it maintains an upper confidence bound $U$ for each policy, and balances between exploration and exploitation. Theorem~\ref{thm:regretGeneral} below shows that Algorithm~\ref{alg-policy-UCB} always gets a value similar to that of the best policy in $\Pi$ up to an additive factor of $\tilde{O}\left(\sqrt{\abs{\Pi}T}+\abs{\Pi}\right)$. The theorem follows directly from Theorem $3$ from \citet{JIA2021} by having policies as arms and returns as rewards.

\begin{restatable}{theorem}{thmregretGeneral}\label{thm:regretGeneral}
Let $\Pi$ be a set of policies with the associated parameters $\tilde \tau,\eta$. Let $\pi_1,\dots,\pi_T$ be the policies Algorithm~\ref{alg-policy-UCB} selects. It holds that
{
\footnotesize
\[
\E\left[ \max_{\pi\in\Pi} T \cdot V^{\pi} - \sum_{t=1}^{T}V^{\pi_t}\right] = O(\sqrt{\abs{\Pi}T\log T}+\abs{\Pi}\log T).
\]
}
\end{restatable}
There are two challenges in leveraging Theorem~\ref{thm:regretGeneral}. 
The first challenge is crucial: Notice that Theorem~\ref{thm:regretGeneral} does not imply that Algorithm~\ref{alg-policy-UCB} has a low regret; its only guarantee is w.r.t. the policies in $\Pi$ received as an input. As the number of policies is infinite, our success will depend on our ability to characterize a ``good'' set of policies $\Pi$. The second challenge is technical: Even if we find such $\Pi$, we still need to characterize the associated $\tilde \tau$ and $\eta$. This is precisely what we do in Section~\ref{sec:oneUser} and \ref{sec:twoTypes}.

\begin{algorithm}[tb]
\begin{algorithmic}[1] %
\STATE \textbf{Input}:  set of policies $\Pi$, number of users %
$T$, $\tilde{\tau}, \eta$ 
\STATE \textbf{Initialize:} ~$ \forall \pi \in~\Pi : U_0(\pi)\leftarrow \infty, n(\pi)=0$
\FOR{user $t\leftarrow 1,\dots,T$} 
\STATE{Execute $\pi_{t}$ such that $\pi_t\in\argmax_{\pi\in \Pi} U_{t-1}(\pi)$ and receive return $\hat{V}^{\pi_{t}}[n(\pi_t)]\leftarrow \sum_{j=1}^{N^{\pi_t}(t)}r_{t,j}(\pi_{t,j})$} 
\STATE $n(\pi_t)\leftarrow n(\pi_t)+1$
\IF {$n(\pi_t)<8\eta \ln T$}
 \STATE Update {\footnotesize $U_t(\pi_t)=\frac{\sum_{i=1}^{n(\pi_t)}\hat{V}^{\pi_t}[i]}{n(\pi_t)}+\frac{8\sqrt {\eta }\cdot \tilde{\tau}\ln T}{n(\pi_t)}$}
 \ELSE
 \STATE Update {\footnotesize $U_t(\pi_t)=\frac{\sum_{i=1}^{n(\pi_t)}\hat{V}^{\pi_{t}}[i]}{n(\pi_t)}+\sqrt {\frac{8 \tilde{\tau}^2\ln T}{n(\pi_t)}}$}
 \ENDIF
\ENDFOR
\caption{UCB-based algorithm with hybrid radii: UCB-Hybrid \cite{JIA2021}}\label{alg-policy-UCB}
\end{algorithmic}
\end{algorithm}

\section{Single User Type}\label{sec:oneUser}
In this section, we focus on the special case of a single user type, i.e., $M=1$. For notational convenience, 
since we only discuss single-type users, 
we associate each category $a\in [K]$
with its two unique parameters $\P_a:= \P_{a,1},\LeavProb_a:=\LeavProb_{a,1}$ 
and refer to them as scalars rather than vectors. In addition, 
We use the notation  
$N_a$ for the random variable representing 
the number of iterations until a random user departs 
after being recommended by $\pi^a$, the fixed-arm policy that recommends category $a$ in each iteration. 

To derive a regret bound for single-type users, we use two main lemmas: Lemma~\ref{lem:valueSingle}, which shows the optimal policy is fixed, and Lemma~\ref{lemma:subExpTaub}, which shows that returns of fixed-arm policies are sub-exponential and calculate 
 their corresponding parameters. 
These lemmas allow us to use Algorithm~\ref{alg-policy-UCB} with a policy set $\Pi$ that contains all the fixed-arm policies, and derive a $\tilde{O}(\sqrt{T})$ regret bound. For brevity, we relegate all the proofs to the Appendix.

To show that there exists a category $a^*\in[K]$ for which $\pi^{a^*}$ is optimal, we rely on the assumption that all the users have the same type (hence we drop the type subscripts $t$), and as a result the rewards of each category $a\in [K]$ have an expectation that depends on a single parameter, namely $\E[r(a)]=\P_a$. 
Such a category $a^*\in[K]$ does not necessarily have the maximal click-probability nor the minimal departure-probability, but rather an optimal combination of the two (in a way, this is similar to the knapsack problem, where we want to maximize the reward while having as little weight as possible). 
We formalize it in the following lemma. 
 
\begin{restatable}{lemma}{lemValueSingle}\label{lem:valueSingle}
A policy $\pi^{a^*}$ is optimal if \[
a^*%
\in 
\argmax_{a\in [K]}\frac{\P_{a}}{\LeavProb_a(1-\P_a)}.
\]
\end{restatable}
As a consequence of this lemma, the planning problem for single-type users is trivial---the solution is a fixed-arm policy $\pi^{a^*}$ given in the lemma. However, without access to the model parameters, identifying $\pi^{a^*}$ requires learning. 
We proceed with a simple observation regarding the random number of iterations obtained by executing a fixed-arm policy. The observation would later help us show  that the return of any fixed-arm policy is sub-exponential.
\begin{observation}\label{obs:XisGeo}
For every $a\in[K]$ and every $\LeavProb_{a}>0$,
the random variable $N_a$ follows a geometric distribution 
with success probability parameter $\LeavProb_a[1-\P_a]\in(0,1-\epsilon]$.
\end{observation}
 
Using Observation~\ref{obs:XisGeo} and previously known results (stated~{\ifnum\Includeappendix=1{as Lemma~\ref{lem:geoIsSubExp} 
}\else{}\fi}in the appendix), we 
show that $N_a$ is sub-exponential for all $a \in [K]$.
Notice that return realizations are always upper bounded by the trajectory length; this implies that returns are also sub-exponential. However, to use the regret bound of Algorithm~\ref{alg-policy-UCB}, we need information regarding the parameters $(\tau^2_a,b_a)$  for every policy $\pi^a$. We provide this information in the following Lemma~\ref{lemma:subExpTaub}.

\begin{restatable}{lemma}{lemmaSubExpTaub}\label{lemma:subExpTaub}
For each category $a\in[K]$, the centred 
random variable $V^{\pi^a}-\E[V^{\pi^a}]$ 
is sub-exponential with parameters $(\tau_a^2,b_a)$,
such that
\footnotesize
\[
\tau_a=b_a= -\frac{8e}{\ln(1-\LeavProb_a(1-\P_a))}.
\]
\end{restatable}
\begin{proofS}
We rely on the equivalence between the subexponentiality of a random variable and the bounds on its moments (Property~\ref{it:moments} in Definition~\ref{def:subExp}).
We bound the expectation of the return $V^{\pi^a}$, and use Minkowski's and Jensen's inequalities to show~{\ifnum\Includeappendix=1{in Lemma~\ref{lemma:ValuesAreExp} 
}\else{}\fi}that $\E[|V^{\pi^a}-\E[V^{\pi^a}]|^p])^{1/p}$ is upper bounded by $-4/\ln(1-\LeavProb_a(1-\P_a))$ for every  $a\in[K]$ and $p\geq 1$. 
Finally, we apply a normalization trick and bound the Taylor series of $\E[\exp(\gamma(V^{\pi^a}-\E[V^{\pi^a}]))]$ to obtain the result.
\end{proofS}

An immediate consequence of Lemma~\ref{lemma:subExpTaub} is that the parameters $\tilde{\tau}=8e/\ln(\frac{1}{1-\epsilon})$ and $\eta=1$ are valid upper bounds for $\tau_a$ and $b_a/\tau_a^2$ for each $a\in [K]$ (I.e., $\forall a\in [K]:\;\tilde{\tau} \geq \tau_a$ and $\eta\geq  b_a^2/\tau_a^2$). We can now derive a regret bound using Algorithm~\ref{alg-policy-UCB} and Theorem~\ref{thm:regretGeneral}.

\begin{restatable}{theorem}{thmregretOneType}\label{thm:regretOneType}
For single-type users ($M=1$), running Algorithm~\ref{alg-policy-UCB} with $\Pi=\{\pi^a:a\in [K]\}$ and $\tilde{\tau}=\frac{8e}{\ln(\frac{1}{1-\epsilon})}$, $\eta=1$ achieves an expected regret of at most
\[
\E[R_T]=O(\sqrt{KT\log T}+K\log T).
\]
\end{restatable}

\section{Two User Types and Two  Categories}\label{sec:twoTypes}

In this section, we consider cases with
two user types ($M=2$), two categories ($K=2$) and departure-probability $\LeavProb_{a,\tau}=1$ for every category $a\in[K]$ and type $\tau \in [M]$. Even in this relatively simplified setting, where users leave after the first ``no-click'', planning is essential. To see this, notice that the event of a user clicking on a certain category provides additional information about the user, which can be used to tailor better recommendations; hence, algorithms that do not take this into account may suffer a linear regret. In fact, this is not just a matter of the learning algorithm at hand, but rather a failure of all fixed-arm policies; there are instances where all fixed-arm policies yield high regret w.r.t. the baseline defined in Equation~\eqref{eq:regredef}. Indeed, this is what the example in Section~\ref{subsec:example} showcases. Such an observation suggests that studying the optimal planning problem is vital.

In Section~\ref{subsec:Planning2x2}, we introduce the partially observable MDP formulation of departing bandits along with notion of \textit{belief-category walk}. We use this notion to provide a closed-form formula for policies' expected return, which we use extensively later on. Next, in Section~\ref{subsec:2x2Characterizing} we characterize the optimal policy, and show that we can compute it in constant time relying on the closed-form formula. This is striking, as generally computing optimal POMDP policies is computationally intractable since, e.g., the space of policies grows exponentially with the horizon. Conceptually, we show that there exists an optimal policy that depends on a belief threshold: It recommends one category until the posterior belief of one type, which is monotonically increasing,  crosses the threshold, and then it recommends the other category. Finally, in Section~\ref{subsec:2x2Learning} we leverage all the previously obtained results to derive a small set of threshold policies of size $O(\ln T)$ with corresponding sub-exponential parameters. Due to Theorem~\ref{thm:regretGeneral}, this result implies a $\tilde O(\sqrt{T})$ regret.
\subsection{Efficient Planning}\label{subsec:Planning2x2}
To recap,  we aim to find the optimal policy when the click-probability matrix and the prior over user types are known. Namely, given an instance in the form of $\langle \mathbf{P}, \mathbf q \rangle$, our goal is to efficiently find the optimal policy.

For planning purposes, the problem can be modeled by an episodic POMDP, $\langle S,[K], O, \text{Tr}, \P, \Omega, \q, O \rangle$. A set of states, $S=[M]\cup \{\perp\}$ that comprises all types $[M]$, along with a designated absorbing state $\perp$ suggesting that the user departed (and the episode terminated). $[K]$ is the set of the actions (categories). $O=\{stay,depart\}$ is the set of possible observations. The transition and observation functions, $\text{Tr}:S\times [K]\rightarrow S$ and $\Omega: S\times [K] \rightarrow O$ (respectively) satisfy $\text{Tr}(\perp|i, a)=\Omega(depart|i,a)=1-\P_{i,a}$ and  $\text{Tr}(i|i,a)=\Omega(stay|i,a)=\P_{i,a}$ for every type $i\in[M]$ and action $a\in[K]$. Finally, $\P$ is the expected reward matrix, and $\mathbf q$ is the initial state distribution over the $M$ types.

When there are two user types and two categories, the click-probability matrix is given by Table~\ref{tab:example2} where we note that the prior on the types holds $\q_y = 1 - \q_x$, thus can be represented by a single parameter~$\q_x$.
\begin{remark}\label{remark:max_entry_1_x}
Without loss of generality, 
we assume that $\P_{1,x} \geq \P_{2,x}, \P_{1,y}, \P_{2,y}$
since one could always permute the matrix to obtain such a structure.
\end{remark}
Since the return and number of iterations for the same policy is independent of the user index, we drop the subscript $t$ in the rest of this subsection and use . 

\begin{table}[H]
    \centering
    \begin{tabular}
        {|p{2cm}||p{2cm}|p{2cm}|}
 \hline
   & Type $x$  & Type $y$\\
 \hline
 Category $1$   & $\P_{1,x}$ & $\P_{1,y}$ \\ 
 Category $2$ &  $\P_{2,x}$ &  $\P_{2,y}$ \\ 
 \hline
 Prior &$\q_x$ & $\q_y=1-\q_x$\\
 \hline
    \end{tabular}
    \caption{Click probabilities for two user types and two categories.} %
    \label{tab:example2}
\end{table}

As is well-known in the POMDP literature~\cite{kaelbling1998planning}, 
the optimal policy $\pi^*$ and its expected return are functions of belief states that represent  the probability of the state at each time.  In our setting, the states are the user types. We denote by $b_j$ the belief that the state is (type) $x$ at iteration $j$. Similarly, $1-b_j$ is the belief that the state is (type) $y$ at iteration $j$. Needless to say, once the state $\perp$ is reached, the belief over the type states $[M]$ is %
irrelevant, as users do not come back. Nevertheless, we neglect this case as our analysis does not make use it.

We now describe how to compute the belief. At iteration $j=1$, the belief state is set to be $b_1 = \mathbb{P}\left(state= x\right) = \q_x$. At iteration $j>1$, upon receiving a positive reward $r_j=1$, the belief is updated from $b_{j-1} \in [0,1]$ to  
\begin{align}\label{eq:beliefupdate}
    b_j(b_{j-1} ,a,1)
    =\frac{b_{j-1} \cdot \P_{a,x}}{b_{j-1} \cdot \P_{a,x}+\P_{a,y} (1-b_{j-1} )},
\end{align}
where we note that in the event of no-click, 
the current user departs the system, i.e., we move to the absorbing state $\perp$. For any policy $\pi: [0,1] \to \{1,2\}$ that maps a belief to a category, its expected return satisfies the Bellman equation,
\begin{align*}
    \E[V^{\pi}(b)]&=\left(b \P_{\pi(b), x} +
    (1-b) \P_{\pi(b), y}\right) \cdot\\
    &\qquad  (1+\E[V^{\pi}(b'(b,\pi(b),1))]).
\end{align*}
To better characterize the expected return, we introduce the following notion of belief-category walk. 
\begin{definition}[Belief-category walk]\label{defn:trajectory}
Let $\pi:[0,1]\rightarrow \{1,2\}$ be any policy. The sequence 
\[
b_1,a_1 = \pi(b_1), b_2, a_2 = \pi(b_2),\dots
\]
is called the \textit{belief-category walk}. Namely, it is the induced walk of belief updates and categories chosen by $\pi$, given all the rewards are positive ($r_j=1$ for every $j\in \mathbb{N}$).
\end{definition} 

Notice that every policy induces a single,  well-defined and deterministic belief-category %
walk %
(recall that we assume departure-probabilities satisfy $\LeavProb_{a,\tau}=1$ for every $a\in [K], \tau\in[M]$). Moreover, given any policy $\pi$, 
the trajectory of every user recommended by $\pi$ 
is fully characterized by belief-category walk 
clipped at $b_{N^\pi(t)}, a_{N^\pi(t)}$.

In what follows, we  derive a closed-form expression for the expected return as a function of $b$, the categories chosen by the policy, and the click-probability matrix.
\begin{restatable}{theorem}{thmValueUpperBound}\label{thm:valueUpperBound}
For every policy $\pi$ and an initial belief $b \in [0,1]$, 
the expected return is given by 
\[
\E[V^{\pi}(b)]=\sum_{i=1}^\infty b\cdot \P_{1,x}^{m_{1,i}}\cdot \P_{2,x}^{m_  {2,i}}+(1-b)\P_{1,y}^{m_{1,i}}\cdot \P_{2,y}^{m_{2,i}},
\]
where $m_{1,i}:=|\{a_j=1,j\leq i\}|$ and $m_{2,i}:=|\{a_j=2,j\leq i\}|$ are calculated based on the belief-category walk $b_1,a_1,b_2,a_2,\dots$ induced by $\pi$.
\end{restatable}

\subsection{Characterizing the Optimal Policy}\label{subsec:2x2Characterizing}
Using Theorem~\ref{thm:valueUpperBound}, we show that the planning problem can be solved in $O(1)$. To arrive at this conclusion, we perform a case analysis over the following three structures of the click-probability matrix~$\P$:
\begin{itemize}
    \item \textit{Dominant Row}, where $\P_{1,y}\geq \P_{2,y}$; %
    \item \textit{Dominant Column}, where $\P_{2,x}\geq \P_{2,y}>\P_{1,y}$; %
    \item \textit{Dominant Diagonal}, where $\P_{1,x}\geq \P_{2,y}>\P_{1,y},\P_{2,x}$. %
\end{itemize}
Crucially, any matrix $\P$ takes exactly one of the three structures. Further, since $\P$ is known in the planning problem, identifying the structure at hand takes $O(1)$ time. Using this structure partition, we characterize the optimal policy. 
\paragraph{Dominant Row}
We start by considering the simplest structure, in which the Category $1$ is preferred by both types of users: Since $\P_{1,y}\geq \P_{2,y}$
and $\P_{1,x} \geq \P_{2,x}, \P_{1,y}, \P_{2,y}$ (Remark~\ref{remark:max_entry_1_x}),
there exists a dominant row, i.e., Category $1$. 
\begin{restatable}{lemma}{lemmaStochasticDominantRow}\label{lemma:stochastic-dominant-row}
For any instance such that $\P$ has a dominant row $a$, the fixed policy $\pi^a$ is an optimal policy.
\end{restatable}
As expected, if Category $1$ is dominant then the policy that always recommends Category $1$ is optimal.
\paragraph{Dominant Column}
In the second structure we consider the case where there is no dominant row, and that the column of type $x$ is dominant, i.e.,  $\P_{1,x}\geq \P_{2,x}\geq \P_{2,y}> \P_{1,y}$. In such a case, which is also the one described in the example in Section~\ref{subsec:example}, it is unclear what the optimal policy would be since none of the categories dominates the other.

Surprisingly, we show that the optimal policy can be of only one form: Recommend Category $2$ for some time steps (possibly zero) and then always recommend Category $1$.  
To identify when to switch from Category $2$ to Category $1$, 
one only needs to compare four expected returns. 

\begin{restatable}{theorem}{thmStochasticDominantColumn}\label{thm:stochastic-dominant-column}
For any instance such that $\P$ has a dominant column, one of the following four policies is optimal:
\[
\pi^1,\pi^2,\pi^{2:\floor{N^*}},\pi^{2:\ceil{N^*}},
\]
where $N^*=N^*(\P,\q)$ is a constant, %
and $\pi^{2:\floor{N^*}}$ ($\pi^{2:\ceil{N^*}}$) 
stands for recommending Category $2$ until iteration 
$\floor{N^*}$ ($\ceil{N^*}$) and then switching to Category $1$.
\end{restatable}
The intuition behind the theorem is as follows. If the prior tends towards type $y$, we might start with recommending Category $2$ (which users of type $y$ are more likely to click on). But after several iterations, and as long as the user stays, the posterior belief $b$ increases since $\P_{2,x} > \P_{2,y}$ (recall Equation~\eqref{eq:beliefupdate}). Consequently, since type $x$ becomes more probable, and since $\P_{1,x} \geq \P_{2,x}$, the optimal policy recommends the best category for this type, i.e., Category $1$. 
For the exact expression of $N^*$, we refer the reader to~{\ifnum\Includeappendix=1{Appendix~\ref{appendix:SDC}}\else{the appendix}\fi}.

Using Theorem~\ref{thm:valueUpperBound}, 
we can compute the expected return for each of the four policies in $O(1)$, showing that we can find the optimal policy when $\P$ has a column in $O(1)$.

\paragraph{Dominant Diagonal}
In the last structure,
we consider the case where there is no dominant row (i.e., $\P_{2,y}> \P_{1,y}$) nor a dominant column (i.e., $\P_{2,y}> \P_{2,x}$). %
At first glance, this case is more complex than the previous two, 
since none of the categories and none of the types dominates the other one. 
However, we uncover that the optimal policy can be either always recommending Category $1$ or  always recommending Category $2$. Theorem~\ref{thm:stochastic-diagonal-dominance} summarizes this result.

\begin{restatable}{theorem}{thmStochasticDiagonalDominance}\label{thm:stochastic-diagonal-dominance}
For any instance such that $\P$ has a dominant diagonal, 
either $\pi^1$ or $\pi^2$ is optimal.
\end{restatable}

With the full characterization of the optimal policy derived in this section (for all the three structures), 
we have shown that the optimal policy can be computed in $O(1)$. 

\subsection{Learning: UCB-based regret bound}\label{subsec:2x2Learning}
In this section, we move from the planning task to the learning one. Building on the results of previous sections, we know that there must exist a threshold policy---a policy whose belief-category walk has a finite prefix of one category, and an infinite suffix with the other category---which is optimal. However, there can still be infinitely many such policies. To address this problem, we first show how to reduce the search space for approximately optimal policies with negligible additive factor to a set of $\abs{\Pi}=O(\ln(T))$ policies. Then, we derive the parameters $\tilde \tau$ and $\eta$ required for Algorithm~\ref{alg-policy-UCB}. As an immediate consequence, we get a sublinear regret algorithm for this setting.
We begin with defining threshold policies.

\begin{definition}[Threshold Policy]
A policy $\pi$ is called an $(a,h)$-threshold policy if there exists an number $h\in \mathbb N\cup \{0\}$ in $\pi$'s belief-category walk such that 
\begin{itemize}
    \item $\pi$ recommends category $a$ in iterations $j\leq h$, and
    \item $\pi$ recommends category $a'$ in iterations $j>h$,
\end{itemize} 
for $a,a'\in \{1,2\}$ and $a \neq a'$.
\end{definition}
For instance, the policy $\pi^1$ that always recommends Category 1 is the $(2,0)$-threshold policy, as it recommends Category 2 until the zero'th iteration (i.e., never recommends Category 2) and then Category 1 eternally. Furthermore, the policy $\pi^{2:\floor{N^*}}$ introduced in Theorem~\ref{thm:stochastic-dominant-column} is the $(2,\floor{N^*})$-threshold policy.

Next, recall that the chance of departure in every iteration is greater or equal to $\epsilon$, since we assume $\max_{a,\tau} \P_{a,\tau} \leq 1-\epsilon$. Consequently, the probability that a user will stay beyond $H$ iterations is exponentially decreasing with $H$. We could use high-probability arguments to claim that it suffices to focus on the first $H$  iterations, but without further insights this would yield $\Omega(2^H)$ candidates for the optimal policy. Instead, we exploit our insights about threshold policies. 

Let $\Pi_H$ be the set of all $(a,h)$-threshold policies for $a\in \{1,2\}$ and $h\in [H]\cup\{0\}$. Clearly, $\abs{\Pi_H}=2H+2$. Lemma~\ref{lem:OptVsFiniteHorizonOpt} shows that the return obtained by the best policy in $\Pi_H$ is not worse than %
that of the optimal policy $\pi^*$ by a negligible factor.
\begin{restatable}{lemma}{lmOptVsFiniteHorizonOpt}\label{lem:OptVsFiniteHorizonOpt}
For every $H\in \mathbb N$, it holds that
\[
\E\left[V^{\pi^*}-\max_{\pi\in \Pi_H} V^{\pi}\right]\leq \frac{1}{2^{O(H)}}.
\]
\end{restatable}
Before we describe how to apply Algorithm~\ref{alg-policy-UCB}, we need to show that returns of all the policies in $\Pi_H$ are sub-exponential.
In Lemma~\ref{lem:ThresholdSubExp}, we show that $V^{\pi}$ is $(\tau^2,b)$-sub-exponential for every threshold policy $\pi\in\Pi_H$, and provide bounds for both $\tau$ and $b^2/\tau^2$.
\begin{restatable}{lemma}{lemThresholdSubExp}\label{lem:ThresholdSubExp}
Let $\tilde{\tau}=\frac{8e}{\ln(\frac{1}{1-\epsilon})}$ and $\eta=1$. For every threshold policy $\pi\in\Pi_H$,
the centred random variable $V^{\pi}-\E[V^{\pi}]$ is $(\tau^2,b)$-sub-exponential with $(\tau^2,b)$ satisfying $\tilde{\tau}\geq \tau$ and $\eta \geq b^2/\tau^2$.
\end{restatable}

We are ready to wrap up our solution for the learning task proposed in this section. Let $H=\Theta(\ln T)$, $\Pi_H$ be the set of threshold policies characterized before, and let $\tilde{\tau}$ and $\eta$ be constants as defined in Lemma~\ref{lem:ThresholdSubExp}.
\begin{theorem}
Applying Algorithm~\ref{alg-policy-UCB} with $\Pi_H, T, \tilde \tau, \eta$ on the class of two-types two-categories instances considered in this section always yields an expected regret of
\[
\E[R_T]\leq O(\sqrt{T}\ln T).
\]
\end{theorem}
\begin{proof}
It holds that
{
\begin{align*}
    &\E[R_T] = \E\left[TV^{\pi^*}- \sum_{t=1}^{T}V^{\pi_t} \right]\\
    &=\E\left[TV^{\pi^*}-\max_{\pi\in \Pi_H}T V^{\pi}\right]+\E\left[\max_{\pi\in \Pi_H} T V^{\pi}- \sum_{t=1}^{T}V^{\pi_t} \right] \\
    &\overset{(*)}{\leq} \frac{T}{2^{O(H)}}+O(\sqrt{HT\log T}+H\log T)=O(\sqrt{T}\ln T),
\end{align*}
}%
where $(*)$ follows from Theorem~\ref{thm:regretGeneral} and Lemma~\ref{lem:OptVsFiniteHorizonOpt}. 
Finally, setting $H=\Theta(\ln T)$ yields the desired result.
\end{proof}

\section{Conclusions and Discussion}\label{sec:discussion}
This paper introduces a MAB model in which the recommender system influences both the rewards accrued \textit{and} the length of interaction. We dealt with two classes of problems: A single user type with general departure probabilities (Section~\ref{sec:oneUser}) and the two user types, two categories where each user departs after her first no-click (Section~\ref{sec:twoTypes}). For each problem class, we started with analyzing the planning task, then characterized a small set of candidates for the optimal policy, and then applied Algorithm~\ref{alg-policy-UCB} %
to achieve sublinear regret.

In the appendix, we also consider a third class of problems: Two categories, multiple user types ($M\geq 2)$ where user departs with their first no-click. We use the closed-form expected return derived in Theorem~\ref{thm:valueUpperBound} to show how to use dynamic programming to find approximately optimal planning policies. We formulate the problem of finding an optimal policy for a finite horizon $H$ in a recursive manner. Particularly, we show how to find a $\nicefrac{1}{2^{O(H)}}$ additive approximation in run-time of $O(H^2)$. Unfortunately, this approach cannot assist us in the learning task. Dynamic programming relies on skipping sub-optimal solutions to sub-problems (shorter horizons in our case), but this happens on the fly; thus, we cannot a-priori define a small set of candidates like what  Algorithm~\ref{alg-policy-UCB} requires. More broadly, we could use this dynamic programming approach for more than two categories, namely for $K\geq 2$, but then the run-time becomes $O(H^K)$. 

There are several interesting future directions. First, achieving low regret for the setup in Section~\ref{sec:twoTypes} with $K\geq 2$. We suspect that this class of problems could enjoy a solution similar to ours, where candidates for optimal policies are mixing two categories solely. 
Second, achieving low regret for the setup in Section~\ref{sec:twoTypes} with uncertain departure (i.e., $\LeavProb\neq 1$). Our approach fails in such a case since we cannot use belief-category walks; these are no longer deterministic. Consequently, the closed-form formula is much more complex and optimal planning becomes more intricate. These two challenges are left open for future work.

\section*{Acknowledgement}
LL is generously supported by an Open Philanthropy AI Fellowship. LC is  supported by Ariane de Rothschild Women Doctoral Program. ZL thanks the Block Center for Technology and Society; Amazon AI; PwC USA via the Digital Transformation and Innovation Center; and the NSF: Fair AI Award IIS2040929 for supporting ACMI lab’s research on the responsible use of machine learning. This project has received funding from the European Research Council (ERC) under the European Union’s Horizon 2020 research and innovation program (grant agreement No. 882396), by the Israel Science Foundation (grant number 993/17), Tel Aviv University Center for AI and Data Science (TAD), and the Yandex Initiative for Machine Learning at Tel Aviv University.

\bibliography{aaai22}

{\ifnum\Includeappendix=1{
\newpage
\appendix

\onecolumn

\section{
Extension: Planning Beyond Two User Types}\label{sec:planningManyTypes}
In this section, we treat the planning task with two categories ($K=2$) but potentially many types (i.e., $M\geq 2$). For convenience, we formalize the results in this section in terms of $M=2$, but the results are  readily extendable for the more general $2 \times M$ case. We derive an almost-optimal planning policy via dynamic programming, and then explain why it cannot be used for learning as we did in the previous section.

For reasons that will become apparent later on, we define by $V_H^\pi$ as the return of a policy $\pi$ until the $H$'s iteration. Using Theorem~\ref{thm:valueUpperBound}, we have that
\[
\E[V_H^{\pi}(b)]=\sum_{i=1}^H b\cdot \P_{1,x}^{m_{1,i}}\cdot \P_{2,x}^{m_  {2,i}}+(1-b)\P_{1,y}^{m_{1,i}}\cdot \P_{2,y}^{m_{2,i}},
\]
where $m_{1,i}:=|\{a_j=1,j\leq i\}|$ and $m_{2,i}:=|\{a_j=2,j\leq i\}|$ are calculated based on the belief-category walk $b_1,a_1,b_2,a_2,\dots$ induced by $\pi$. Further, let $\tilde \pi^*$ denote the policy maximizing $V_H$.

Notice that there is a bijection from $H-$iterations policies to $(m_{1,i},m_{2,i})_{i=1}^H$; hence, we can find $\tilde \pi^*$ by finding the arg max of the expression on the right-hand-side of the above equation, in terms of $(m_{1,i},m_{2,i})_{i=1}^H$. Formally, we want to solve the integer linear programming (ILP),
\begin{equation}\label{eq:ilp}
\begin{array}{ll@{}ll}
&\text{maximize} &\displaystyle\sum_{i=1}^H b\cdot \P_{1,x}^{m_{1,i}}\cdot \P_{2,x}^{m_  {2,i}}+(1-b)\P_{1,y}^{m_{1,i}}\cdot \P_{2,y}^{m_{2,i}} \\
&\text{subject to } &  m_{a,i}=\displaystyle\sum\limits_{l=1}^i z_{a,l} \text{ for }  a\in \{1,2\}, i\in [H],\\
&  &  z_{a,i} \in \{0,1\} \text{ for }  a\in \{1,2\}, i\in [H], \\
&  &  z_{1,i}+z_{2,i} = 1 \text{ for }   i\in [H].\\
\end{array}
\end{equation}
Despite that this problem involves integer programming, we can solve it using dynamic programming in $O\left( H^2 \right)$ runtime. Notice that the optimization is over a subset of binary variables $(z_{1,i},z_{2,i})_{i=1}^H$. Let $Z^H$ be the set of feasible solutions of the ILP, and similarly let $Z^h$ denote set of prefixes of length $h\leq H$ of $Z^H$. 

For any $h\in [H]$ and $\z\in Z^h$, define
\[
D^h(\z)\defeq \sum_{i=1}^{h} b\cdot \P_{1,x}^{m_{1,i}}\cdot \P_{2,x}^{m_  {2,i}}+(1-b)\P_{1,y}^{m_{1,i}}\cdot \P_{2,y}^{m_{2,i}},
\]
where $m_{a,i}=\sum_{l=1}^i z_{a,l} \text{ for }  j\in \{1,2\}, i\in [h]$ as in the ILP.

Consequently, solving the ILP is equivalent to maximizing $D^H$ over the domain $Z^H$.

Next, for any $h\in[H]$ and two integers $c_1,c_2$ such that $c_1+c_2=h$, define
\begin{equation}\label{eq:dtilde}
\tilde D^{h}(c_1,c_2) \defeq \max_{\substack{\z\in Z^{h},\\m_{1,h}(\z)=c_1 \\m_{2,h}(\z)=c_2}} D^{h}(\z). 
\end{equation}
Under this construction, $\max_{c_1,c_2} \tilde D^H (c_1,c_2)$ over $c_1,c_2$ such that $c_1+c_2=H$ is precisely the value of the ILP. 

Reformulating Equation~\eqref{eq:dtilde} for ${h>1}$,
{
\thinmuskip=1mu
\medmuskip=1mu plus 1mu minus 1mu
\thickmuskip=1mu plus 1mu
\footnotesize
\begin{align*}\label{eq:dynamic1}
\tilde D^{h}(c_1,c_2) & 
=\max_{\substack{z_1,z_2\in \{0,1\} \\ z_1+z_2=1}} \left\{\tilde D^{h-1}(c_1-z_1,c_2-z_2)+\alpha(c_1,c_2)\right\},
\end{align*}}%
where $\alpha(m_1,m_2) \defeq b\cdot x_1^{m_1}\cdot x_2^{m_2}+(1-b)y_1^{m_1}\cdot y_2^{m_2}$. 
For every $h$, there are only $h+1$ possible values $\tilde D^h$ can take: All the ways of dividing $h$ into non-negative integers $c_1$ and $c_2$; therefore, having computed $\tilde D^{h-1}$ for all $h$ feasible inputs, we can compute $\tilde D^h(c_1,c_2)$ in $O(h)$. Consequently, computing $\max_{c_1,c_2} \tilde D^H (c_1,c_2)$, which is precisely the value of the ILP in \eqref{eq:ilp}, takes $O(H^2)$ run-time. Moreover, the policy $\tilde \pi^*$ can be found using backtracking. We remark that an argument similar to Lemma~\ref{lem:OptVsFiniteHorizonOpt} implies that $\E[V^{\pi^*}-V^{\tilde \pi^*}]\leq \frac{1}{2^{O(H)}}$; hence, $\tilde \pi^*$ is almost optimal.

To finalize this section, we remark that this approach could also work for $K>2$ categories. Naively, for a finite horizon $H$, there are $K^H$ possible policies. The dynamic programming procedure explain above makes the search operate in run-time of $O(H^K)$. The run-time, exponential in the number of categories but polynomial in the horizon, is feasible when the number of categories is small.

\section{
Experimental Evaluation}\label{sec:experimental}
For general real-world datasets, we propose a scheme to construct semi-synthetic problem instances with many arms and many user types, using rating data sets with multiple ratings per user. 
We exemplify our scheme on the MovieLens Dataset  \citet{Harper15MovieLens}.
As a pre-processing step, we set movie genres to be the categories of interest, 
select a subset of categories $|A|$ of size $k$ (e.g., sci-fi, drama, and comedy), and select the number of user types, $m$. Remove any user who has not provided a rating for at least one movie from each category $a\in A$. When running the algorithm, randomly draw users from the data, and given a recommended category $a$, suggest them a random movie which they have rated, and set their click probability to $1-r$, where $r\in [0,1]$ is their normalized rating of the suggested movie.

\section{UCB Policy for Sub-exponential Returns}

An important tool for analyzing sub-exponential random variables is Bernstein’s Inequality, which is a concentration inequality for sub-exponential random variables (see, e.g., \citet{JIA2021}). Being a major component of the regret analysis for Algorithm~\ref{alg-policy-UCB}, we state it here for  convenience.
\begin{lemma}(Bernstein’s Inequality) Let a random variable $X$ be sub-exponential with parameters $(\tau^2,b)$. Then for every $v\geq 0$:
\[ 
\Pr[|X-\E[X]|\geq v]\leq
\begin{cases} 
      2\exp(-\frac{v^2}{2\tau^2}) & v\leq \frac{\tau^2}{b} \\
      2\exp(-\frac{v}{2b}) & else
   \end{cases}.
\]
\end{lemma}

\section{Single User Type: Proofs from Section \ref{sec:oneUser}}
To simplify the proofs, we use the following notation: For a fixed-arm policy $\pi^a$, we use $V_j^{\pi^a}$ to denote its return from iteration $j$ until the user departs. Namely,
\[
V_j^{\pi^a}=\sum_{i=j}^{N^{\pi^a}}\P_a
\]
Throughout this section, we will use the following Observation.
\begin{observation}
For every policy $\pi$ and iteration $j$,
\[
\E[V^{\pi}_j]=\P_{\pi_j}(1+ \E[V^{\pi}_{j+1}])+(1-\LeavProb_{\pi_j})(1-\P_{\pi_j})\E[V^{\pi}_{j+1}]=\E[V^{\pi}_{j+1}](1-\LeavProb_{\pi_j}(1-\P_{\pi_j}))+\P_{\pi_j}.
\]
\end{observation}
\lemValueSingle*
\begin{proof}
First, recall that every POMDP has an optimal Markovian  policy which is  deterministic (we refer the reader to Section~\ref{subsec:Planning2x2} for full formulation of the problem as POMDP).  Having independent rewards and a single state implies that there exists $\mu^*\in \mathbb{N}$ such that $\E[V^*_j]=\mu^*$ for every $j\in \mathbb{N}$ (similarly to standard MAB problems, there exists a fixed-arm policy which is optimal).\\
Assume by contradiction that the optimal policy $\pi^{a^*}$ holds
\[
a^*\notin \argmax_{a\in [k]}\frac{\P_{a}}{\LeavProb_a(1-\P_a)}.
\]
Now, notice that
\[
\E[V^{\pi^{a'}}]=\E[V^{\pi^{a'}}_1]=%
\E[V^{\pi^{a'}}_2](1-\LeavProb_{a'}(1-\P_{a'}))+\P_{a'}
\]
Solving the recurrence relation and summing the geometric series we get
\[
\E[V^{\pi^{a'}}]=\P_{a'} \sum_{j=0}^\infty (1-\LeavProb_{a'}(1-\P_{a'}))^{j}=\frac{\P_{a'}}{\LeavProb_{a'}(1-\P_{a'})}.
\]
Finally, 
\[
a^*\notin \argmax_{a\in [k]}\frac{\P_{a}}{\LeavProb_a(1-\P_a)},
\]
yields that any fixed-armed policy,  $\pi^{a'}$ such that 
\[
a'\in \argmax_{a\in [k]}\frac{\P_{a}}{\LeavProb_a(1-\P_a)}
\]
holds $\E[V^{\pi^{a'}}]>\E[V^{\pi^{a^*}}]$, a contradiction to the optimality of $\pi^{a^*}$ .
\end{proof}
\begin{restatable}{lemma}{lemmaValuesAreExp}\label{lemma:ValuesAreExp}
For each $a\in[k]$, the centered random return $V^{\pi^a}-\E[V^{\pi^a}]$ 
is sub-exponential with parameter $C_2=-4/\ln(1-\LeavProb_a(1-\P_a))$.
\end{restatable}%
In order to show that returns of fixed-arm policies are sub-exponential random variables, we first show that the number of iterations of users recommended by fixed-arm policies is also a sub-exponential. For this purpose, we state here a lemma that implies that every geometric r.v. is a sub-exponential r.v.. The proof of the next lemma appears, e.g., in \citet{hillar2018maximum} (Lemma $4.3$).
\begin{lemma}\label{lem:geoIsSubExp}
Let $X$ be a geometric random variable with parameter $r \in (0,1)$, so that:
\begin{equation*}
\Pr[X = x] = (1-r)^{x-1} \, r, \quad x \in \mathbb{N}.
\end{equation*}
Then $X$ satisfies Property \eqref{it:moments} from Definition ~\ref{def:subExp}. Namely, $X$ is sub-exponential with parameter $C_2=-2/\ln(1-r)$. Formally, 
\[
    \forall p \geq 0:  (\E[|X|^p])^{1/p} \leq -\frac{2}{\ln(1-r)}p.
\]
\end{lemma}
The lemma above and Observation \ref{obs:XisGeo} allow us to deduce that the variables $N_a$ are sub-exponential in the first  part of the following Corollary (the case in which $\LeavProb_a=0$ follows immediately from definition.). The second part of the lemma follows directly from the equivalences between Properties   (\ref{it:moments}) and (\ref{it:tails}) in Definition ~\ref{def:subExp}.
\begin{corollary}\label{cor:NaIsSubExp}
For each $a\in[K]$, the number of iterations a user recommended by $\pi^a$ stays within the system, $N_a$, is sub-exponential with parameter $C_2^a=-2/\ln(1-\LeavProb_a(1-\P_a))$. In addition, there exist constants $C^a_1>0$ for every $a\in[K]$ such that 
\[
\forall a\in[K],\;v \geq 0: \Pr[|N_a| > v]
\leq  \exp({1-\frac{v}{C_1^a}}).
\]
\end{corollary}
The next Proposition \ref{prop:helperExpectationBound} is used for the proof of Lemma \ref{lemma:ValuesAreExp}.
\begin{proposition}\label{prop:helperExpectationBound}
For every $a\in [K]$,
\[
|\E[V^{\pi^a}]|\leq \frac{-2}{\ln(1-\LeavProb_a(1-\P_a))}
\]
\end{proposition}
\begin{proof}
First, notice that 
\[
(1-\LeavProb_a(1-\P_a))\ln(1-\LeavProb_a(1-\P_a))>(1-\LeavProb_a(1-\P_a))\frac{-\LeavProb_a(1-\P_a)}{1-\LeavProb_a(1-\P_a)}=-\LeavProb_a(1-\P_a)>-2\LeavProb_a(1-\P_a),
\]
where the first inequality is due to $\frac{x}{1+x}
\leq \ln(1+x)$ for every $x\geq -1$. Rearranging, 
\begin{equation}\label{eq:helperExpectationBound}
\frac{1-\LeavProb_a(1-\P_a)}{\LeavProb_a(1-\P_a)}<\frac{-2}{\ln(1-\LeavProb_a(1-\P_a))}.
\end{equation}
For each user, the realization of $V^{\pi^a}$ is less or equal to the realization of $N_a-1$ for the same user (as users provide negative feedback in their last iteration); hence,
\[
|\E[V^{\pi^a}]|=\E[V^{\pi^a}]\leq \E[N_a]-1=\frac{1}{\LeavProb_a(1-\P_a)}-1=\frac{1-\LeavProb_a(1-\P_a)}{\LeavProb_a(1-\P_a)}<\frac{-2}{\ln(1-\LeavProb_a(1-\P_a))}.
\]
\end{proof}
We proceed by showing that returns of fix-armed policies satisfy Property  (\ref{it:tails}) from Definition ~\ref{def:subExp}.
\lemmaValuesAreExp*

\begin{proof}
We use Property (\ref{it:tails}) from Definition ~\ref{def:subExp} to derive that $V^{\pi^a}$ is also sub-exponential. This is true since the tails of $V^{\pi^a}$ satisfy that for all $v\geq 0$,
\[
\Pr[|V^{\pi^a}|> v]\leq \Pr[|N_a| > v+1]\leq \Pr[|N_a| > v]
\leq_{(\ref{it:tails})}\exp({1-\frac{v}{C_1}}),
\]
where the first inequality follows since $|N_a| > v+1$ is a necessary condition for $|V^{\pi^a}|  > v$, and the last inequality follows from  Corollary \ref{cor:NaIsSubExp}. Along with Definition ~\ref{def:subExp}, we conclude that
\begin{equation}\label{eq:valC2}
    \E[|V^{\pi^a}|^p]^{1/p}\leq -2/\ln(1-\LeavProb_a(1-\P_a)) p.
\end{equation}
Now, applying  Minkowski's inequality and then Jensen's inequality (as $f(z)=z^p,g(z)=|z|$ are convex for every $p\geq 1$) we get
\[
(\E[|V^{\pi^a}-\E[V^{\pi^a}]|^p])^{1/p} \leq 
\E[|V^{\pi^a}|^p]^{1/p} + \E[\E[|V^{\pi^a}|]^p]^{1/p}\leq 
\E[|V^{\pi^a}|^p]^{1/p} + |\E[V^{\pi^a}]|.
\]
Using Proposition \ref{prop:helperExpectationBound} and Inequality \eqref{eq:valC2}, we get
\[
\E[|V^{\pi^a}|^p]^{1/p} + |\E[V^{\pi^a}]|\leq \frac{-2}{\ln(1-\LeavProb_a(1-\P_a))}+\frac{1}{\LeavProb_a(1-\P_a)}-1\leq \frac{-4}{\ln(1-\LeavProb_a(1-\P_a))}
\]
Hence $V^{\pi^a}-\E[V^{\pi^a}]$ is sub-exponential with parameter $C_2=-4/\ln(1-\LeavProb_a(1-\P_a))$.
\end{proof}

\lemmaSubExpTaub*
\begin{proof}
Throughout this proof, we will use the sub-exponential norm, $||\cdot||_{\psi_1}$, which is defined as
\[
||Z||_{\psi_1}=\sup_{p\geq 1} \frac{(\E[|Z|^p])^{1/p}}{p}.
\]
Let 
\[
X=\frac{V^{\pi^a}-\E[V^{\pi^a}]}{||V^{\pi^a}-\E[V^{\pi^a}]||_{\psi_1}}, \quad y=\gamma\cdot ||V^{\pi^a}-\E[V^{\pi^a}]||_{\psi_1}.
\]
We have that
\begin{equation}\label{eq:TaubSubExp1}
||X||_{\psi_1}=||\frac{V^{\pi^a}-\E[V^{\pi^a}]}{||V^{\pi^a}-\E[V^{\pi^a}]||_{\psi_1}}||_{\psi_1}=1.
\end{equation}
Let $\gamma$ be such that $|\gamma|< 1/b_a= -\frac{\ln(1-\LeavProb_a(1-\P_a))}{8e}$. From Lemma \ref{lemma:ValuesAreExp} we conclude that
\[
|\gamma|=\big|\frac{y}{ ||V^{\pi^a}-\E[V^{\pi^a}]||_{\psi_1}}\big|\leq  -\frac{\ln(1-\LeavProb_a(1-\P_a))}{8e}=\frac{1}{2e}\cdot \frac{1}{||V^{\pi^a}-\E[V^{\pi^a}]||_{\psi_1}};
\]
hence, $|y|<\frac{1}{2e}$.\\
Summing the geometric series, we get
\begin{equation}\label{eq:TaubSubExp2}
\sum_{p=2}^{\infty}(e|y|)^p=\frac{e^2y^2}{1-e|y|}<2e^2y^2
\end{equation}
In addition, notice that $yX=\gamma(V^{\pi^a}-\E[V^{\pi^a}])$.\\
Next, from the Taylor series of $\exp(\cdot)$ we have
\[
 \E[\exp(\gamma(V^{\pi^a}-\E[V^{\pi^a}]))]=\E[\exp(yX)]= 1+y\E[x]+\sum_{p=2}^{\infty} \frac{y^p\E[X^p]}{p!}.  
\]
Combining the fact that $\E[X]=0$ and (\ref{eq:TaubSubExp1}) to the above,
\[
1+y\E[x]+\sum_{p=2}^{\infty} \frac{y^p\E[X^p]}{p!}\leq 1+\sum_{p=2}^{\infty} \frac{y^p p^p}{p!}.
\]
By applying $p! \geq (\frac{p}{e})^p$ and (\ref{eq:TaubSubExp2}), we get
\[
1+\sum_{p=2}^{\infty} \frac{y^p p^p}{p!}\leq 1+\sum_{p=2}^{\infty}(e|y|)^p\leq 1+2e^2y^2\leq \exp(2e^2y^2)=\exp(2e^2 (\gamma\cdot ||V^{\pi^a}-\E[V^{\pi^a}]||_{\psi_1})^2),
\]
where the last inequality is due to $1+x\leq e^x$.

Note that $||V^{\pi^a}-\E[V^{\pi^a}]||_{\psi_1}\leq -\frac{4}{\ln(1-\LeavProb_a(1-\P_a))})^2)$. Ultimately, 
\[
\E[\exp(\gamma(V^{\pi^a}-\E[V^{\pi^a}]))]
\leq \exp\big(2e^2\gamma^2( -\frac{4}{\ln(1-\LeavProb_a(1-\P_a))})^2\big)=
 \exp\big(\frac{1}{2}\gamma^2( -\frac{8e}{\ln(1-\LeavProb_a(1-\P_a))})^2\big).
\]
This concludes the proof of the lemma.
\end{proof}

\clearpage
\section{Two User Types and Two  Categories: Proofs from Section \ref{sec:twoTypes}}
\subsection{Planning when $K = 2$}

\thmValueUpperBound*
\begin{proof}
Let $\beta^{\pi}_{i}(b):= b\cdot {\P_{1,x}}^{m_{1,i}}\cdot {\P_{2,x}}^{m_{2,i}}+(1-b){\P_{1,y}}^{m_{1,i}}\cdot {\P_{2,y}}^{m_{2,i}}$. 
We will prove that for every policy $\pi$ and every belief $b$, we have that $\E[V_H^{\pi_a}(b)]=\sum_{i=1}^H \beta^{\pi}_{i}(b)$ by a backward induction over $H$.

For the base case, consider $H=1$. We have that 
\[
\E[V_1^{\pi}(b_1)]=b_1\cdot \P_{{a_1},x}+(1-b)\P_{a_1,y}=b\cdot {\P_{1, x}}^{m_{1,1}}\cdot {\P_{2,x}}^{m_{2,1}}+(1-b){\P_{1,y}}^{m_{1,1}}\cdot {\P_{2,y}}^{m_{2,1}}=\beta^{\pi}_{1}(b)
\]
as $m_{a,1}=\mathbb{I}[a_1=a]$.

For the inductive step, assume that $\E[V_{H-1}^{\pi}({b})]=\sum_{i=1}^{H-1} \beta^{\pi}_i({b})$ for every ${b}\in[0,1]$. We need to show that $\E[V_{H}^{\pi}(b)]=\sum_{i=1}^{H} \beta^{\pi}_i(b)$ for every $b\in[0,1]$.

Indeed,
\begin{align*}
\E[V_H^{\pi}(b_1)]&=\beta^{\pi}_1(b_1)(1+\E[V_{H-1}^{\pi}(b'(b_1,a_1,liked))])\\
&=\beta^{\pi}_1(b_1)(1+\E[V_{H-1}^{\pi}(b_2)])\\
&=\beta^{\pi}_1(b_1)(1+\sum_{i=2}^{H-1} \beta^{\pi}_{i}(b_2))\\
&=\sum_{i=1}^{H}\beta^{\pi}_{i}(b_1),    
\end{align*}
where the second to last equality is due to the induction hypothesis and the assumption that $\pi$ is a deterministic stationary policy.
The proof completes by realizing that $\E[V^\pi(b)] = \lim_{H\to\infty} \E[V_H^\pi(b)] = \lim_{H\to \infty}\sum_{i=1}^{H}\beta^{\pi}_{i}(b)=\sum_{i=1}^{\infty}\beta^{\pi}_{i}(b)$, since the sum is finite and has positive summands.
\end{proof}

\subsection{Dominant Row (DR)}
\lemmaStochasticDominantRow*
\begin{proof}
We will show that for every iteration $j$, no matter what were the previous topic recommendations were, selecting topic $1$ rather than topic $2$ can only increase the value.

Let $\pi$ be a stationary policy such that $\pi(b_j)=2$. 
Changing it into a policy $\pi^j$ that is equivalent to $\pi$ for all iterations but iteration $j+1$ in which it recommends topic $1$ can only improve the value.

Since $\P_{1,x},\P_{2,x},\P_{1,y},\P_{2,y}\geq 0$, $\P_{1,x}-\P_{2,x}\geq0$, $b,1-b\geq 0$ and this structure satisfies $\P_{2,y}- \P_{1,y}\leq 0$, we get that for every $\bar{m}_{1,j},\bar{m}_{2,j},n_{1,j},n_{2,j}\in \mathbb{N}$ and for every $b$,
\[
b\cdot \P_{1,x}^{\bar{m}_{1,j}+n_{1,j}}\cdot \P_{2,x}^{\bar{m}_{2,j}+n_{2,j}}(\P_{1,x}-\P_{2,x})\geq (1-b)\P_{1,y}^{\bar{m}_{1,j}+n_{1,j}}\cdot \P_{2,y}^{\bar{m}_{2,j}+n_{2,j}}(\P_{2,y}-\P_{1,y});
\]
thus,
\[
b\cdot \P_{1,x}^{\bar{m}_{1,j}+1+n_{1,j}}\cdot \P_{2,x}^{\bar{m}_{2,j}+n_{2,j}}+(1-b)\P_{1,y}^{\bar{m}_{1,j}+1+n_{1,j}}\cdot \P_{2,y}^{m_{2,j}+n_{2,j}}\geq
\]
\[
b\cdot \P_{1,x}^{\bar{m}_{1,j}+n_{1,j}}\cdot \P_{2,x}^{\bar{m}_{2,j}+1+n_{2,j}} +(1-b)\P_{1,y}^{\bar{m}_{1,j}+n_{1,j}}\cdot \P_{2,y}^{\bar{m}_{2,j}+1+n_{2,j}}.
\]
Hence for every time step $j+1$, choosing topic $1$ instead of topic $2$ leads to increased value of each of the summation element $b\cdot \P_{1,x}^{m_{1,i}}\cdot \P_{2,x}^{m_  {2,i}}+(1-b)\P_{1,y}^{m_{1,i}}\cdot \P_{2,y}^{m_{2,i}}$ such that $m_{1,i}=\bar{m}_{1,j}+n_{1,j}\geq \bar{m}_{1,j}$ and $m_{2,i}=\bar{m}_{2,j}+n_{2,j}\geq \bar{m}_{2,j}$.
We deduce that 
\[
\E[V^{\pi^j}(b)]\geq \E[V^{\pi}(b)].
\]
\end{proof}

\subsection{Dominant Column (DC)}
\label{appendix:SDC}
Before proving the main theorem (Theorem~\ref{thm:stochastic-dominant-column}), 
we prove two auxiliary lemmas.

 \begin{restatable}{lemma}{lemma:keyLemmaStochasticColumnDominance}\label{lemma:key-lemma-stochastic-column-dominance}
For $\P$ with a DC structure, if a policy $\pi$ is optimal then it recommends topic $1$ for all iteration $j' \geq j+1$ such that 
\begin{align}\label{eq:suffCondForOptimal12}
&\sum_{i=j+1}^{\infty} \; \P_{1,x}^{m_{1,i}}\P_{2,x}^{m_{2,i}}> \quad  \sum_{i=j+1}^{\infty} \frac{1-b}{b}\cdot \frac{\P_{2,y}-\P_{1,y}}{\P_{1,x}-\P_{2,x}}\cdot \frac{\P_{2,x}}{\P_{2,y}}\P_{1,y}^{m_{1,i}}\P_{2,y}^{m_{2,i}}.
\end{align}
\end{restatable}

\begin{proof}%
First, assume by contradiction that there exists an optimal policy $\pi$ that recommends topic $2$ in iteration $j+1$ such that (\ref{eq:suffCondForOptimal12}) holds.

Let $\pi^j$ be the policy that is equivalent to $\pi$ but recommend topic $1$ instead of topic $2$ in iteration $j+1$. Since $\pi$ and $\pi^j$ recommends the same topic until iteration $j$, along with the optimality of $\pi$, we have
\[
\E[V^{\pi^j}(b)]- \E[V^{\pi}(b)]=\E[V_{j+1}^{\pi^j}(b)]-\E[V_{j+1}^{\pi}(b)]\leq 0.
\]
Expending the above equation,
\[
\sum_{i=j+1}^\infty  b\cdot \P_{1,x}^{m_{1,i}^{\pi}+1} \cdot \P_{2,x}^{m_{2,i}^{\pi}-1}+(1-b)\P_{1,y}^{m_{1,i}^{\pi}+1}\cdot \P_{2,y}^{m_{2,i}^{\pi}-1}-
\big(
\sum_{i=j+1}^\infty  b\cdot \P_{1,x}^{m_{1,i}^{\pi}} \cdot \P_{2,x}^{m_{2,i}^{\pi}}+(1-b)\P_{1,y}^{m_{1,i}^{\pi}}\cdot \P_{2,y}^{m_{2,i}^{\pi}}
\big)\leq 0
\]
\[
\sum_{i=j+1}^\infty  b\cdot \P_{1,x}^{m_{1,i}^{\pi}} \cdot \P_{2,x}^{m_{2,i}^{\pi}}(\frac{\P_{1,x}}{\P_{2,x}}-1)\leq
\sum_{i=j+1}^\infty (1-b)\P_{1,y}^{m_{1,i}^{\pi}}\cdot \P_{2,y}^{m_{2,i}^{\pi}}(1-\frac{\P_{1,y}}{\P_{2,y}})
\]
\[
\frac{b(\P_{1,x}-\P_{2,x})}{\P_{2,x}}\sum_{i=j+1}^\infty  \cdot \P_{1,x}^{m_{1,i}^{\pi}} \cdot \P_{2,x}^{m_{2,i}^{\pi}}\leq
\frac{(1-b)(\P_{2,y}-\P_{1,y})}{\P_{2,y}} \sum_{i=j+1}^\infty \P_{1,y}^{m_{1,i}^{\pi}}\cdot \P_{2,y}^{m_{2,i}^{\pi}}
\]
\[
\sum_{i=j+1}^\infty  \P_{1,x}^{m_{1,i}^{\pi}} \cdot \P_{2,x}^{m_{2,i}^{\pi}}\leq
\frac{1-b}{b}\cdot \frac{\P_{2,x}}{\P_{2,y}}\cdot \frac{\P_{2,y}-\P_{1,y}}{\P_{1,x}-\P_{2,x}} \sum_{i=j+1}^\infty \P_{1,y}^{m_{1,i}^{\pi}}\cdot \P_{2,y}^{m_{2,i}^{\pi}},
\]
which is a contradiction to (\ref{eq:suffCondForOptimal12}).

For the second part of the lemma, assume that condition (\ref{eq:suffCondForOptimal12}) holds for some iteration $j+1\in \mathbb{N}$ and some optimal policy $\pi$; 
hence, $\pi(b,m^{\pi}_{1,j},m^{\pi}_{2,j})=1$ and we have $m^{\pi}_{1,j+1}=m^{\pi}_{1,j}+1$ and $m^{\pi}_{2,j+1}=m^{\pi}_{2,j}$. Exploiting this fact, we have that
\[
 \sum_{i=j+2}^{\infty} \; \P_{1,x}^{m^{\pi}_{1,i}}\P_{2,x}^{m^{\pi}_{2,i}}=\sum_{i=j+1}^{\infty} \; \P_{1,x}^{m^{\pi}_{1,i}+1}\P_{2,x}^{m^{\pi}_{2,i}}=\P_{1,x}\sum_{i=j+1}^{\infty} \; \P_{1,x}^{m^{\pi}_{1,i}}\P_{2,x}^{m^{\pi}_{2,i}}>(\ref{eq:suffCondForOptimal12}),
 \]
implying 
 \begin{align*}
&\P_{1,x}\sum_{i=j+1}^{\infty} \frac{1-b}{b}\cdot \frac{\P_{2,y}-\P_{1,y}}{\P_{1,x}-\P_{2,x}}\cdot \frac{\P_{2,x}}{\P_{2,y}}\P_{1,y}^{m^{\pi}_{1,i}}\P_{2,y}^{m^{\pi}_{2,i}}\\
>&(\P_{1,x}\geq \P_{1,y})
 \P_{1,y}\sum_{i=j+1}^{\infty} \frac{1-b}{b}\cdot \frac{\P_{2,y}-\P_{1,y}}{\P_{1,x}-\P_{2,x}}\cdot \frac{\P_{2,x}}{\P_{2,y}}\P_{1,y}^{m^{\pi}_{1,i}}\P_{2,y}^{m^{\pi}_{2,i}}\\
 =&
 \sum_{i=j+1}^{\infty} \frac{1-b}{b}\cdot \frac{\P_{2,y}-\P_{1,y}}{\P_{1,x}-\P_{2,x}}\cdot \frac{\P_{2,x}}{\P_{2,y}}\P_{1,y}^{m^{\pi}_{1,i}+1}\P_{2,y}^{m^{\pi}_{2,i}}\\
 =&
 \sum_{i=j+2}^{\infty} \frac{1-b}{b}\cdot \frac{\P_{2,y}-\P_{1,y}}{\P_{1,x}-\P_{2,x}}\cdot \frac{\P_{2,x}}{\P_{2,y}}\P_{1,y}^{m^{\pi}_{1,i}}\P_{2,y}^{m^{\pi}_{2,i}}.
 \end{align*}
\end{proof}

An immediate consequence of Lemma~\ref{lemma:key-lemma-stochastic-column-dominance} is the following corollary.

\begin{corollary}\label{cor:switching-stochastic-column-dominance}
For any DC-structured $\P$ and every belief $b\in[0,1]$, the optimal policy $\pi$ first recommends topic $2$ for at most
\begin{equation*}
\text{argmin}_{N} \sum_{i=1}^{N} \; \P_{2,x}^{m^{\pi}_{2,i}}>  \frac{1-b}{b}\cdot \frac{\P_{2,y}-\P_{1,y}}{\P_{1,x}-\P_{2,x}}\cdot \frac{\P_{2,x}}{\P_{2,y}}
\sum_{i=1}^{N}\P_{2,y}^{m^{\pi}_{2,i}}
\end{equation*}
times, and then recommends topic $1$ permanently. In addition, $N\in \mathbb{N}$ since $\P_{2,x}>\P_{2,y}$.
\end{corollary}

\thmStochasticDominantColumn*
\begin{proof}
Due to Theorem~\ref{thm:valueUpperBound} and Corollary~\ref{cor:switching-stochastic-column-dominance}, we can write the expected value of a policy as a function of~ $N$ when $\P$ has a DC structure:
\begin{align}\label{eq:func of N}
   \E[V^{\pi_N}(b)]&=\sum_{i=1}^\infty b\cdot \P_{1,x}^{m_{1,i}}\cdot \P_{2,x}^{m_  {2,i}}+(1-b)\P_{1,y}^{m_{1,i}}\cdot \P_{2,y}^{m_{2,i}} \nonumber \\
   & = \sum_{i=1}^N b\cdot \P_{2,x}^i +(1-b)\P_{2,y}^i + \sum_{i=N+1}^\infty b\cdot \P_{2,x}^N\cdot \P_{1,x}^{i-N}+(1-b)\P_{2,y}^N\cdot \P_{1,y}^{i-N}  \nonumber \\
   &= b\cdot \frac{\P_{2,x}(\P_{2,x}^N-1)}{\P_{2,x}-1}+(1-b)\cdot \frac{\P_{2,y}(\P_{2,y}^N-1)}{\P_{2,y}-1} + b\cdot \P_{2,x}^N\cdot \sum_{i=1}^\infty \P_{1,x}^{i}+(1-b)\P_{2,y}^N\cdot\sum_{i=1}^\infty \P_{1,y}^{i}  \nonumber \\
   &= b\cdot \frac{\P_{2,x}(\P_{2,x}^N-1)}{\P_{2,x}-1}+(1-b)\cdot \frac{\P_{2,y}(\P_{2,y}^N-1)}{\P_{2,y}-1} + b\cdot \P_{2,x}^N\cdot \frac{\P_{1,x}}{1-\P_{1,x}}+(1-b)\P_{2,y}^N\frac{\P_{1,y}}{1-\P_{1,y}} \nonumber \\
   &= \P_{2,x}^N\cdot b\big(\frac{\P_{2,x}}{\P_{2,x}-1}+\frac{\P_{1,x}}{1-\P_{1,x}}\big)+\P_{2,y}^N(1-b)\big(\frac{\P_{2,y}}{\P_{2,y}-1} + \frac{\P_{1,y}}{1-\P_{1,y}}\big)
   +\frac{b \P_{2,x}}{1-\P_{2,x}}+\frac{(1-b) \P_{2,y}}{1-\P_{2,y}}.
\end{align}

Equation~\eqref{eq:func of N} could be cast as $c_1 \cdot \P_{2,x} ^N +c_2 \P_{2,y}^N +c_3(\P_{2,x},\P_{2,y})$ for \textbf{positive} $c_1$, \textbf{negative} $c_2$ and positive $c_3$. Let $f:\mathbb R \leftarrow \mathbb R$ be the continuous function such that $f(N)=c_1 \cdot \P_{2,x} ^N +c_2 \P_{2,y}^N +c_3(\P_{2,x},\P_{2,y})$. 

We take the derivative w.r.t. $N$ to find the saddle point of $f$:
\[
 \frac{d}{dN}f=c_1 \cdot \ln \P_{2,x} \cdot \P_{2,x} ^{N} +c_2 \ln \P_{2,y}\cdot \P_{2,y}^{N} =0,
\]
which suggests the saddle point of $f$ is
\[
\tilde{N}=\frac{\ln\big(-\frac{c_2 \ln \P_{2,y}}{ c_1 \ln \P_{2,x}}\big)}{\ln \big(\frac{\P_{2,x}}{\P_{2,y}}\big)}.
\]

Next, set $N^* \defeq \max\{0, \tilde{N}\}$.  Since $f$ has a single saddle point and for every $n\in \mathbb N$ it holds that $f(N)=\E[V^{\pi_N}(b)]$,
to determine the optimal policy, 
one only needs to compare the value $\E[V^{\pi_N}(b)]$ at the boundary points $(N=0, N=\infty)$
and at the closest integers to the saddle point $(N=\floor{N^*}, N=\ceil{N^*})$.
\end{proof}

\subsection{Dominant Diagonal (SD)}
\thmStochasticDiagonalDominance*
\begin{proof}
We prove the following claim by a backward induction over the number of iterations remaining:
For every $k=H-1,\dots 1$ it holds that for every policy $\pi$
and belief $b$, 
\[
\E[V_k^{\pi}(b)] \leq \max\{\E[V_k^{\pi^1}(b)],\E[V_k^{\pi^2}(b)]\}.
\]
First, we notice that when $k=H-1$, 
the only possible policies are $\pi^1$ and $\pi^2$. 
For $k = H-2$,
we prove the statement by contradiction. There are only two ways to selects topics when $k=H-2$:
\begin{align*}
    \pi'_{1:H} = (\pi_{1:H-2}, \underbrace{1}_{{H-1}},
    \underbrace{2}_{{H}})\quad \text{and} \quad  \pi_{1:H}'' = (\pi_{1:H-2}, \underbrace{2}_{H-1},
    \underbrace{1}_{{H}}).
\end{align*}
Let $m_1$ and $m_2$ denote the number of times $\pi$ has played topic $1$ and $2$ till time $H-2$, inclusive. Assume that the policy $\pi'$
is optimal. In particular, it holds that $\E[V^{\pi_1}_k] \leq \E[V^{\pi'}_k]$ and $\E[V^{\pi_2}_k] \leq \E[V^{\pi'}_k]$. We get 
\begin{align}\label{eq:contradiction-eq1H2}
b\P_{1,x}^{m_1}\P_{2,x}^{m_2}\P_{1,x}(\P_{1,x}-\P_{2,x})\leq \P_{1,y}^{m_1}\P_{2,y}^{m_2}(1-b)\P_{1,y}(\P_{2,y}-\P_{1,y}),
\end{align}
and 
\begin{align}\label{eq:contradiction-eq2H2}
\P_{1,y}^{m_1}\P_{2,y}^{m_2}(1-b)(\P_{2,y}-\P_{1,y})(1+\P_{2,y})\leq b\P_{1,x}^{m_1}\P_{2,x}^{m_2}(\P_{1,x}-\P_{2,x})(1+\P_{2,x}).
\end{align} 
As both sides of (\ref{eq:contradiction-eq1H2}) and (\ref{eq:contradiction-eq2H2}) are positive, the multiplication of their left hand sides is smaller than the multiplication of their right hand sides, i.e.,
\begin{align*}
    &\; b\P_{1,x}^{m_1}\P_{2,x}^{m_2}\P_{1,x}(\P_{1,x}-\P_{2,x})\P_{1,y}^{m_1}\P_{2,y}^{m_2}(1-b)(\P_{2,y}-\P_{1,y})(1+\P_{2,y})\\
    \leq& \; \P_{1,y}^{m_1}\P_{2,y}^{m_2}(1-b)\P_{1,y}(\P_{2,y}-\P_{1,y})b\P_{1,x}^{m_1}\P_{2,x}^{m_2}(\P_{1,x}-\P_{2,x})(1+\P_{2,x})
\end{align*}
Dividing both sides by $b\P_{1,x}^{m_1}\P_{2,x}^{m_2}(\P_{1,x}-\P_{2,x})\P_{1,y}^{m_1}\P_{2,y}^{m_2}(1-b)(\P_{2,y}-\P_{1,y}) >0$, we obtain
\[
\P_{1,x}(1+\P_{2,y})\leq \P_{1,y}(1+\P_{2,x}),
\]
which is a contradiction as $\P_{1,x}>\P_{1,y}$ and $1+\P_{2,y}>1+\P_{2,x}$.

Now, assume that the policy $\pi''$
is optimal. In particular, it holds that $\E[V^{\pi^1}_k] \leq \E[V^{\pi''}_k]$ and $\E[V^{\pi^2}_k] \leq \E[V^{\pi''}_k]$. We get 
\begin{align}\label{eq:contradiction-eq3H2}
\P_{1,x}^{m_1}\P_{2,x}^{m_2}b(\P_{1,x}-\P_{2,x})(1+\P_{1,x})\leq \P_{1,y}^{m_1}\P_{2,y}^{m_2}(1-b)(1+\P_{1,y})(\P_{2,y}-\P_{1,y}),
\end{align}
and 
\begin{align}\label{eq:contradiction-eq4H2}
\P_{1,y}^{m_1}\P_{2,y}^{m_2}(1-b)\P_{2,y}(\P_{2,y}-\P_{1,y})\leq \P_{1,x}^{m_1}\P_{2,x}^{m_2}b\P_{2,x}(\P_{1,x}-\P_{2,x}).
\end{align}
As both sides of (\ref{eq:contradiction-eq3H2}) and (\ref{eq:contradiction-eq4H2}) are positive, the multiplication of their left hand sides is smaller than the multiplication of their right hand sides,
\begin{align*}
    &\; \P_{1,x}^{m_1}\P_{2,x}^{m_2}b(\P_{1,x}-\P_{2,x})(1+\P_{1,x})\P_{1,y}^{m_1}\P_{2,y}^{m_2}(1-b)\P_{2,y}(\P_{2,y}-\P_{1,y})\\
    \leq&\; \P_{1,y}^{m_1}\P_{2,y}^{m_2}(1-b)(1+\P_{1,y})(\P_{2,y}-\P_{1,y})\P_{1,x}^{m_1}\P_{2,x}^{m_2}b\P_{2,x}(\P_{1,x}-\P_{2,x}).
\end{align*}
Dividing both sides by $\P_{1,x}^{m_1}\P_{2,x}^{m_2}b(\P_{1,x}-\P_{2,x})\P_{1,y}^{m_1}\P_{2,y}^{m_2}(1-b)(\P_{2,y}-\P_{1,y}) >0$, we obtain
\[
\P_{2,y}(1+\P_{1,x})\leq \P_{2,x}(1+\P_{1,y}),
\]
which is again, a contradiction as $\P_{2,x}<\P_{2,y}$ and $1+\P_{1,y}<1+\P_{1,x}$.

For $H \geq 3$,
we prove the statement by contradiction.
Suppose not, i.e., the optimal policy $\pi$
switch recommended topic at least once. 
Let $t$ denote the time step where $\pi$ switch for the last time. 
We first consider the case where $\pi$ has switched from 
topic $2$ to topic $1$ at time $t$. 
More specifically, we have
\begin{align*}
    \pi_{1:H} = (\pi_{1:t-2}, \underbrace{2}_{\pi_{t-1}},
    \underbrace{1}_{\pi_{t}}, \underbrace{ 1, \ldots, 1}_{\pi_{t+1:H-1}}, \underbrace{1}_{\pi_H}).
\end{align*}
Consider another policy $\tilde{\pi}$ (that behaves the same as $\pi$ except at time step $t-1$) defined as 
\begin{align*}
    \tilde{\pi}_{1:H} = (\pi_{1:t-2}, \underbrace{2}_{\pi_{t-1}},
    \underbrace{2}_{\pi_{t}}, \underbrace{ 1, \ldots, 1}_{\pi_{t+1:H-1}}, \underbrace{1}_{\pi_H}).
\end{align*}
Let $m_1$ and $m_2$ denote the number of times $\pi$ has recommended topic $1$ and $2$ till (and include) time $t-1$. 
Since $\pi$ is optimal, we have the difference between the value of $\pi$ and $\tilde{\pi}$ to be non-negative, i.e., 
\begin{align}\label{eq:contradiction-eq1}
      \E[V^{{\pi}}_H] - \E[V^{\tilde{\pi}}_H] = \sum_{i=1}^{H-t+1} b \P_{1,x}^{m_1+i-1} \P_{2,x}^{m_2+1} (\P_{1,x}-\P_{2,x})
      + (1-b) \P_{1,y}^{m_1+i-1} \P_{2,y}^{m_2+1} (\P_{1,y}-\P_{2,y}) \geq 0,
\end{align}
where the difference is induced by the discrepancy 
of the two policies from time step $t$ to $H$. 
Consider another policy $\pi'$ (that behaves the same as $\pi$ except at time step $H$) defined as 
\begin{align*}
    \pi'_{1:H} = (\pi_{1:t-2}, \underbrace{2}_{\pi_{t-1}},
    \underbrace{1}_{\pi_{t}}, \underbrace{ 1, \ldots, 1}_{\pi_{t+1:H-1}}, \underbrace{2}_{\pi_H}).
\end{align*}
Since $\pi$ is optimal, we have the difference between the value of $\pi$ and $\pi'$ to be non-negative, i.e., 
\[
\E[V^{{\pi}}_H] > \E[V^{\pi'}_H] \Rightarrow b \P_{1,x}^{m_1+H-t} \P_{2,x}^{m_2} (\P_{1,x}-\P_{2,x})
      > (1-b) \P_{1,y}^{m_1+H-t} \P_{2,y}^{m_2} (\P_{2,y}-\P_{1,y}),
\]
where the difference is induced by the discrepancy of the two policies from time step $H$. 
Multiplying both sides by $\P_{1,y}>0$, we get
\[
\P_{1,y} b \P_{1,x}^{m_1+H-t} \P_{2,x}^{m_2} (\P_{1,x}-\P_{2,x})
     > (1-b) \P_{1,y}^{m_1+H-t+1} \P_{2,y}^{m_2} (\P_{2,y}-\P_{1,y}).
\]
Using $\frac{\P_{1,x}}{\P_{1,y}}>1$, and $\P_{1,y} b \P_{1,x}^{m_1+H-t} \P_{2,x}^{m_2} (\P_{1,x}-\P_{2,x})>0$,
\[
b \P_{1,x}^{m_1+H-t+1} \P_{2,x}^{m_2} (\P_{1,x}-\P_{2,x})
     > (1-b) \P_{1,y}^{m_1+H-t+1} \P_{2,y}^{m_2} (\P_{2,y}-\P_{1,y});
\]
hence,
\begin{align}\label{eq:contradiction-eq3}
      b \P_{1,x}^{m_1+H-t+1} \P_{2,x}^{m_2} (\P_{1,x}-\P_{2,x})
      + (1-b) \P_{1,y}^{m_1+H-t+1} \P_{2,y}^{m_2} (\P_{1,y}-\P_{2,y}) \geq 0.
\end{align}
Next, we construct a new policy $\pi_\text{new}$ 
that outperforms $\pi$.
We let $\pi^\text{new}$ to be the policy defined as below 
\begin{align*}
    {\pi}^\text{new}_{1:H} =  (\pi_{1:t-2}, \underbrace{1}_{\pi_{t-1}},
    \underbrace{1}_{\pi_{t}}, \underbrace{ 1, \ldots, 1}_{\pi_{t+1:H-1}}, \underbrace{1}_{\pi_H}).
\end{align*}
The value difference between $\pi^\text{new}$ and $\pi$ (caused by the discrepancy 
of the two policies from time $t-1$ to $H$) is
\begin{align*}
    \E[V^{{\pi}^\text{new}}_H]  - \E[V^\pi_H] &= 
    \sum_{i=1}^{H-t+1} b \P_{1,x}^{m_1+i-1} \P_{2,x}^{m_2}(\P_{1,x} - \P_{2,x})
    + (1- b) \P_{1,y}^{m_1+i-1} \P_{2,y}^{m_2}(\P_{1,y} - \P_{2,y}) \\
    &+  b \P_{1,x}^{m_1+H-t+1} \P_{2,x}^{m_2}(\P_{1,x} - \P_{2,x})
    + (1- b) \P_{1,y}^{m_1+H-t+1} \P_{2,y}^{m_2}(\P_{1,y} - \P_{2,y})\\
    &> 
     \sum_{i=1}^{H-t+1} b \P_{1,x}^{m_1+i-1} \P_{2,x}^{m_2+1} (\P_{1,x}-\P_{2,x})
      + (1-b) \P_{1,y}^{m_1+i-1} \P_{2,y}^{m_2+1} (\P_{1,y}-\P_{2,y})\\
    &\geq 0,
\end{align*} 
where the first inequality is true because  $\P_{2,x} < \P_{2,y}$, 
$\P_{1,x}-\P_{2,x} >0$ and $\P_{1,y} - \P_{2,y} <0$, therefore for every $1\leq i \leq H-t+1$
\[
b \P_{1,x}^{m_1+i-1} \P_{2,x}^{m_2}(\P_{1,x} - \P_{2,x})(1-\P_{2,x})>0>
   (1- b) \P_{1,y}^{m_1+i-1} \P_{2,y}^{m_2}(\P_{2,y} - \P_{1,y}) (\P_{2,y}-1)
\]
along with (\ref{eq:contradiction-eq3}).
The second inequality follows from (\ref{eq:contradiction-eq1}).
Thus, we have successfully find another policy $\pi^\text{new}_{1:H}$ that differs from $\pi$
and achieves a higher value,
which is a contradiction.

next, we consider the case where $\pi$ has switched from topic $1$ to topic $2$ at time $t$, i.e., 
\begin{align*}
    \pi_{1:H} = (\pi_{1:t-2}, \underbrace{1}_{\pi_{t-1}},
    \underbrace{2}_{\pi_{t}}, \underbrace{ 2, \ldots, 2}_{\pi_{t+1:H-1}}, \underbrace{2}_{\pi_H}).
\end{align*}
Consider another policy $\tilde{\pi}$ (that behaves the same as $\pi$ except at time step $t$) defined as 
\begin{align*}
    \tilde{\pi}_{1:H} = (\pi_{1:t-2}, \underbrace{1}_{\pi_{t-1}},
    \underbrace{1}_{\pi_{t}}, \underbrace{ 2, \ldots, 2}_{\pi_{t+1:H-1}}, \underbrace{2}_{\pi_H}).
\end{align*}
Since $\pi$ is optimal, we have the difference between the value of $\pi$ and $\tilde{\pi}$ to be non-negative, i.e., 
\begin{align}\label{eq:contradiction-eq2}
      \E[V^{{\pi}}_H] - \E[V^{\tilde{\pi}}_H] = \sum_{i=1}^{H-t+1} b \P_{1,x}^{m_1+1} \P_{2,x}^{m_2+i-1} (\P_{2,x}-\P_{1,x})
      + (1-b) \P_{1,y}^{m_1+1} \P_{2,y}^{m_2+i-1} (\P_{2,y}-\P_{1,y}) \geq 0,
\end{align}
where the difference follows from the discrepancy
between the two policies from time step $t$ to $H$. 

Consider another policy $\pi'$ (that behaves the same as $\pi$ except at time step $H$) defined as 
\begin{align*}
    \pi'_{1:H} = (\pi_{1:t-2}, \underbrace{1}_{\pi_{t-1}},
    \underbrace{2}_{\pi_{t}}, \underbrace{ 2, \ldots, 2}_{\pi_{t+1:H-1}}, \underbrace{1}_{\pi_H}).
\end{align*}

Since $\pi$ is optimal, we have the difference between the value of $\pi$ and $\pi'$ to be non-negative, i.e., 
\[
 \E[V^{{\pi}}_H] > \E[V^{\pi'}_H] \Rightarrow (1-b) \P_{1,y}^{m_1} \P_{2,y}^{m_2+H-t} (\P_{2,y}-\P_{1,y})\geq 
 b \P_{1,x}^{m_1} \P_{2,x}^{m_2+H-t} (\P_{1,x}-\P_{2,x}),
\]
where the difference is induced by the discrepancy of the two policies from time step $H$. 
Multiplying both sides by $\P_{2,x}>0$,
\[
\P_{2,x}(1-b) \P_{1,y}^{m_1} \P_{2,y}^{m_2+H-t} (\P_{2,y}-\P_{1,y})\geq 
 b \P_{1,x}^{m_1} \P_{2,x}^{m_2+H-t+1} (\P_{1,x}-\P_{2,x}).
\]
Using $\P_{2,x}(1-b) \P_{1,y}^{m_1} \P_{2,y}^{m_2+H-t} (\P_{2,y}-\P_{1,y})>0$ and $\frac{\P_{2,y}}{\P_{2,x}}\geq 1$, we get
\[
(1-b) \P_{1,y}^{m_1} \P_{2,y}^{m_2+H-t+1} (\P_{2,y}-\P_{1,y})\geq 
 b \P_{1,x}^{m_1} \P_{2,x}^{m_2+H-t+1} (\P_{1,x}-\P_{2,x});
\]
hence,
\begin{align}\label{eq:contradiction-eq4}
      b \P_{1,x}^{m_1} \P_{2,x}^{m_2+H-t+1} (\P_{2,x}-\P_{1,x})
      + (1-b) \P_{1,y}^{m_1} \P_{2,y}^{m_2+H-t+1} (\P_{2,y}-\P_{1,y}) \geq 0.
\end{align}
Again, we will construct a new policy $\pi_\text{new}$ 
that outperforms $\pi$.
We let $\pi^\text{new}$ to be the policy defined as below 
\begin{align*}
    {\pi}^\text{new}_{1:H} =  (\pi_{1:t-2}, \underbrace{2}_{\pi_{t-1}},
    \underbrace{2}_{\pi_{t}}, \underbrace{ 2, \ldots, 2}_{\pi_{t+1:H-1}}, \underbrace{1}_{\pi_H}).
\end{align*}
Now, the value difference between $\pi^\text{new}$ and $\pi$ (caused by the discrepancy 
of the two policies from time $t-1$ to $H$) is
\begin{align*}
    \E[V^{{\pi}^\text{new}}_H]  - \E[V^\pi_H] &= 
    \sum_{i=1}^{H-t+1} \big(b \P_{1,x}^{m_1} \P_{2,x}^{m_2+i-1}(\P_{2,x} - \P_{1,x})
    + (1- b) \P_{1,y}^{m_1} \P_{2,y}^{m_2+i-1}(\P_{2,y} - \P_{1,y})\big) \\
    &+ b \P_{1,x}^{m_1} \P_{2,x}^{m_2+H-t+1}(\P_{2,x} - \P_{1,x})
    + (1- b) \P_{1,y}^{m_1} \P_{2,y}^{m_2+i-1}(\P_{2,y} - \P_{1,y})\\
    &> 
     \sum_{i=1}^{H-t+1} b \P_{1,x}^{m_1+1} \P_{2,x}^{m_2+i-1} (\P_{2,x}-\P_{1,x})
      + (1-b) \P_{1,y}^{m_1+1} \P_{2,y}^{m_2+i-1} (\P_{2,y}-\P_{1,y})\\
    &\geq 0,
\end{align*} 
where the first inequality is true because $\P_{1,y} < \P_{1,x}$, 
$\P_{2,x}-\P_{1,x} <0$ and $\P_{2,y}- \P_{1,y} >0$ and (\ref{eq:contradiction-eq4}), and the second from (\ref{eq:contradiction-eq2}).
Similarly, we have successfully find another policy $\pi^\text{new}_{1:H}$ that differs from $\pi$
and achieves a higher value,
which is a contradiction.

We have covered all cases, so the inductive argument holds. This concludes the proof of the theorem.
\end{proof}

\subsection{UCB-based regret bound}
\lemThresholdSubExp*
\begin{proof}

Let $\gamma$ be such that 

\[
|\gamma|<-\frac{\ln(1-\epsilon)}{8e}\leq \min_{a\in\{1,2\},i\in\{x,y\}}\{-\frac{\ln(1-\LeavProb_{a,i}(1-\P_{a,i}))}{8e}\}=\min_{a\in\{1,2\},i\in\{x,y\}}\{-\frac{\ln(\P_{a,i})}{8e}\}.
\]
Next, we have that
\[
\E[\exp({\gamma(V^{\pi}-\E[V^{\pi}])})]\leq 
\sum_{a\in\{1,2\}}\E[\exp({\gamma(V^{\pi_a}-\E[V^{\pi_a}])})]\big|type(t)\in \argmax_{i\in[1,2]} \P_{a,i}]\cdot \Pr[type(t)\in \argmax_{i\in[1,2]} \P_{a,i}]\leq 
\]
\[
\max_{a\in\{1,2\}}\{\E[\exp({\gamma(\bar{V}^{\pi^a}-\E[\bar{V}^{\pi^a}])})]\}
\]
Where $\bar{V}^{\pi_a}$ is the return for the instance $\langle [1],[2],\q,\bar{\P},\bar{\LeavProb} \rangle$ such that for every $a\in\{1,2\}$: $\bar{\P}_{a,1}=\max_{i\in\{x,y\}}\P_{a,i}$ and $\LeavProb_{a,1}=1$. 

Finally, from Lemma \ref{lemma:subExpTaub} we get
\[
\max_{a\in\{1,2\}}\{\E[\exp({\gamma(\bar{V}^{\pi^a}-\E[\bar{V}^{\pi^a}])})]\}\leq 
\max_{a\in\{1,2\}} \exp((-\frac{8e}{\ln(\bar{\P}_{a,1})})^2\frac{\gamma ^2}{2})= 
\max_{a\in\{1,2\},i\in\{x,y\}}
\exp((-\frac{8e}{\ln(\P_{a,i})})^2\frac{\gamma ^2}{2}).
\]
Choosing
\[
\tau=b=\max_{a\in\{1,2\},i\in\{x,y\}} -\frac{8e}{\ln(\P_{a,i})}
\]
completes the proof as 
\[
\max_{a\in\{1,2\},i\in\{x,y\}} -\frac{8e}{\ln(\P_{a,i})}\leq -\frac{8e}{\ln(1-\epsilon)}=\tilde{\tau} \quad \text{and} \quad \frac{\tau^2}{b^2}=1=\eta.
\]
\end{proof}

\lmOptVsFiniteHorizonOpt*
\begin{proof}
Recall that $V^{\pi}=\sum_{j=1}^{N^\pi}r_{j}(\pi_{j})$, where we drop the dependence on the user index for readability. Formulating differently, for any $H\in \mathbb N$ it holds that
\[
V^{\pi}=\sum_{j=1}^{H} \ind_{j\leq N^{\pi}} \cdot r_{j}(\pi_{j})+\sum_{j=H+1}^{\infty} \ind_{j\leq N^{\pi}} \cdot r_{j}(\pi_{j}).
\]
Using the same representation for $V^{\pi'}$ and taking expectation, we get that
\begin{align*}
    \E\left[ V^{\pi}-V^{\pi'} \right] &\leq
    \E\left[\sum_{j=1}^{H} \ind_{j\leq N^{\pi}} \cdot r_{j}(\pi_{j})-\sum_{j=1}^{H} \ind_{j\leq N^{\pi'}} \cdot r_{j}(\pi_{j}')  \right]+\E\left[ \sum_{j=H+1}^{\infty} \ind_{j\leq N^{\pi}} \cdot r_{j}(\pi_{j}) \right]
    \\
    &\leq 
    0+\E\left[ \sum_{j=H+1}^{\infty} \ind_{j\leq N^{\pi}} \cdot r_{j}(\pi_{j}) \right]= \sum_{j=H+1}^{\infty} \Pr\left( j\leq N^{\pi} \right) r_{j}(\pi_{j}) \\
    &\leq \sum_{j=H+1}^{\infty} (1-\epsilon)^j (1-\epsilon) =  (1-\epsilon)^{H+2} \sum_{j=0}^{\infty} (1-\epsilon)^j \\
    &\leq  (1-\epsilon)^{H} \frac{1}{\epsilon}\leq \frac{e^{-\epsilon H}}{\epsilon} =\frac{1}{2^{O(H)}}.
\end{align*}
\end{proof}

}\fi}

\end{document}